\documentclass[a4]{article}

\usepackage{amsmath,amssymb,ascmac,amsthm}
\usepackage{microtype}
\usepackage[dvipdfmx]{graphicx}
\usepackage{booktabs} % for professional tables

\usepackage{mediabb}
\usepackage{bm}
\usepackage{xcolor}
\usepackage{here}
\usepackage{natbib}
\usepackage{listings}
\usepackage{authblk}
\usepackage{mathtools}
\usepackage{subfig}
\usepackage{url}

\usepackage{accents}
\usepackage{wrapfig}
\usepackage{algpseudocode}
\usepackage{algorithm}
\algdef{SE}[PROCEDURE]{Procedure}{EndProcedure}%
   [2]{\algorithmicprocedure\ \textproc{#1}\ifthenelse{\equal{#2}{}}{}{(#2)}}%
   {\algorithmicend\ \algorithmicprocedure}%
\algdef{SE}[FUNCTION]{Function}{EndFunction}%
   [2]{\algorithmicfunction\ \textproc{#1}\ifthenelse{\equal{#2}{}}{}{(#2)}}%
   {\algorithmicend\ \algorithmicfunction}%
\newcommand{\ubar}[1]{\underaccent{\bar}{#1}}

\def\*#1{\boldsymbol{#1}}

\newtheorem{theo}{Theorem}[section]

\newcommand{\tw}[0]{\textwidth}
\newcommand{\igr}[2]{\includegraphics[clip,width=#1\tw]{#2}}

\newcommand{\argmax}{\mathop{\text{argmax}}}

\newcommand{\eq}[1]{(\ref{#1})}

\newcommand{\lw}[1]{\smash{\lower2.ex\hbox{#1}}}

\newcommand{\veps}{\varepsilon}

\newcommand{\RR}{\mathbb{R}}
\newcommand{\NN}{\mathbb{N}}

\newcommand{\EE}{\mathbb{E}}

\newcommand{\cC}{{\cal C}}
\newcommand{\cD}{{\cal D}}

\newcommand{\cF}{{\cal F}}

\newcommand{\cM}{{\cal M}}
\newcommand{\cN}{{\cal N}}

\newcommand{\cX}{{\cal X}}

\newif\ifdraft %
\draftfalse

\title{Multi-objective Bayesian Optimization using \\ Pareto-frontier Entropy}

\author[1]{Shinya~Suzuki}
\author[2]{Shion~Takeno}
\author[3]{Tomoyuki~Tamura}
\author[4]{Kazuki~Shitara}
\author[5]{Masayuki~Karasuyama}

\affil[1,2,3,5]{Nagoya Institute of Technology}
\affil[3,4,5]{National Institute for Material Science}
\affil[4]{Osaka University}
\affil[4]{Japan Fine Ceramics Center}

\affil[ ]{\textit{suzuki.s.mllab.nit@gmail.com, takeno.s.mllab.nit@gmail.com, tamura.tomoyuki@nitech.ac.jp, shitara@jwri.osaka-u.ac.jp, karasuyama@nitech.ac.jp}}

\date{}

\begin{document}
\maketitle

\begin{abstract}
This paper studies an entropy-based multi-objective Bayesian optimization (MBO).
The \emph{entropy search} is successful approach to Bayesian optimization.
However, for MBO, existing entropy-based methods ignore trade-off among objectives or introduce unreliable approximations.
We propose a novel entropy-based MBO called \emph{Pareto-frontier entropy search} (PFES) by considering the entropy of \emph{Pareto-frontier}, which is an essential notion of the optimality of the multi-objective problem.
Our entropy can incorporate the trade-off relation of the optimal values, and further, we derive an analytical formula without introducing additional approximations or simplifications to the standard entropy search setting. 
We also show that our entropy computation is practically feasible by using a recursive decomposition technique which has been known in studies of the Pareto hyper-volume computation.
Besides the usual MBO setting, in which all the objectives are simultaneously observed, we also consider the ``decoupled'' setting, in which the objective functions can be observed separately.
PFES can easily adapt to the decoupled setting by considering the entropy of the marginal density for each output dimension.
This approach incorporates dependency among objectives conditioned on Pareto-frontier, which is ignored by the existing method.
Our numerical experiments show effectiveness of PFES through several benchmark datasets.
\end{abstract}

% --------------------------------------------------
\section{Introduction}
\label{sec:intro}

\setlength{\textfloatsep}{10pt}
\begin{figure}[t]
 \centering
 %  \igr{.47}{../working_figs/pareto-entropy.pdf}
 \igr{.47}{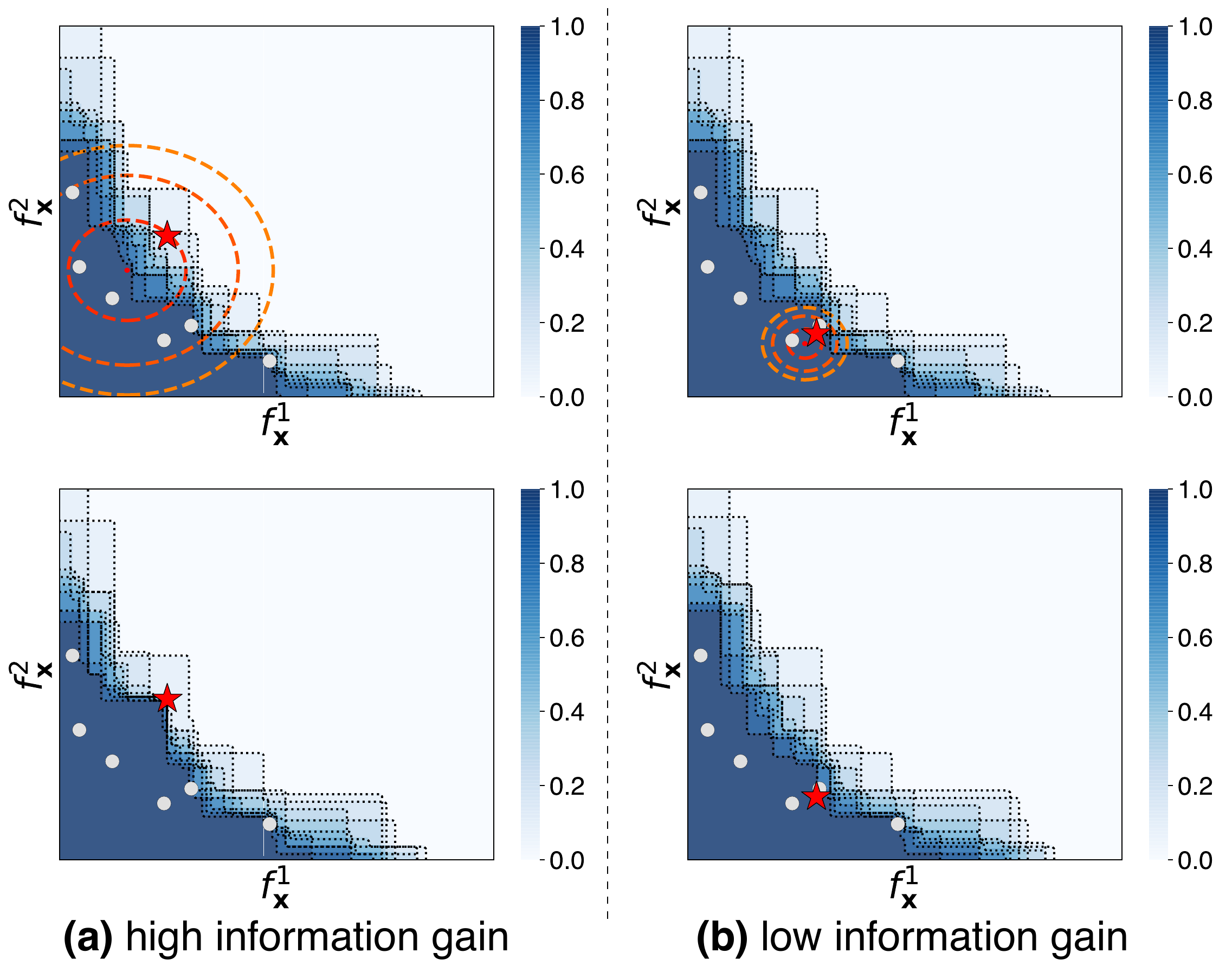}
 \caption{
 Illustrative examples of Pareto-frontier entropy. 
 The blue heatmap represents $p(\*f_{\*x} \preceq \cF^*)$ estimated by $10$ Pareto-frontiers sampled from Gaussian process, and the white points are the observed data.
 The predictive distribution for $\*f_{\*x}$ is illustrated by the nested red circles in the above two plots. 
 Suppose that the red star point is a sample generated by each predictive distribution, and the bottom two plots show $p(\*f_{\*x} \preceq \cF^*)$  after adding the red star point into the training dataset.
 (a) 
 Since the predictive distribution is on around the Pareto-frontier, the mutual information between $\*f_{\*x}$ and $\cF^*$ is high.
 Therefore, when a sample is obtained from the predictive distribution, uncertainty of the Pareto-frontier is drastically reduced in the bottom plot.
 (b) 
 Since the predictive distribution has low variance and is not particularly close to the Pareto-frontier, the mutual information is low in this case.
 Even if a sample is obtained from the predictive distribution, uncertainty of the Pareto-frontier is not largely changed.
 }
 \label{fig:pareto-entropy}
\end{figure}

This paper studies the black-box optimization problem with multiple objective functions.
A variety of engineering problems require optimally designing multiple utility evaluations.
For example, in materials design of the lithium-ion batteries, simultaneously maximizing ion-conductivity and stability is required for practical use.
This type of problems can be formulated as jointly maximizing $L$ unknown functions $f^1(\*x), \ldots, f^L(\*x)$ on some input domain $\cX$, which is called \emph{multi-objective optimization} (MOO) problem.
MOO is often quite challenging because, typically, there does not exist any single optimal option due to the trade-off relation among different objectives.
Further, as in the case of the single-objective black-box optimization, obtaining observations of each objective function is often highly expensive.
For example, in the scientific experimental design such as synthesizing proteins, querying one observation can take more than a day.

A widely accepted optimality criterion in MOO is \emph{Pareto-optimal}.
% 
% A widely accepted approach to MOO is to search a set of \emph{Pareto-optimal} points instead of searching the single best point.
%
For a pair of 
$\*f_{\*x} \coloneqq (f^1(\*x),\ldots,f^L(\*x))^\top$ 
and 
$\*f_{\*x'} \coloneqq (f^1(\*x'),\ldots,f^L(\*x'))^\top$, 
if
$f^l({\*x}) \geq f^l({\*x'})$ for $\forall l \in \{ 1, \ldots, L \}$, 
we say ``$\*f_{\*x}$ dominates $\*f_{\*x'}$''. % and the relation is denoted as $\*f_{\*x} \succeq \*f_{\*x'}$.
If $\*f_{\*x}$ is not dominated by any other $\*f_{\*x'}$ in the domain, $\*f_{\*x}$ is called Pareto-optimal.
\emph{Pareto-frontier} $\cF^*$ is defined as a set of Pareto-optimal $\*f_{\*x}$.
% which is written as $\cF^* \coloneqq \{ \*f_{\*x} \in \cF_{\cX} \mid \*f_{\*x'} \not\succeq \*f_{\*x}, \forall \*f_{\*x'} \in \cF_{\cX} \}$, where $\cF_{\cX}$ is a space created by $\*f_{\*x}$ for $\forall \*x \in \cX$.
%
In other words, if $\*f_{\*x}$ is included in a Pareto-frontier set, there are no alternative $\*x$ which can improve $\*f_{\*x}$ in every objective simultaneously.

In this paper, we focus on the information-based approach to searching Pareto-optimal points with \emph{Bayesian optimization} (BO).
A seminal work on this direction is \emph{predictive entropy search for multi-objective optimization} (PESMO), which defines an acquisition function through the entropy of a set of Pareto-optimal $\*x$ \citep{Hernandez-lobatoa2016-Predictive}.
They showed that the entropy-based acquisition function can achieve the superior performance compared with other types of criteria.
However, PESMO employs an approximation based on \emph{expectation propagation} (EP) \citep{Minka2001-Expectation} because the direct evaluation of their entropy is computationally intractable.
In their EP, a non-Gaussian density is replaced with the Gaussian, and it is difficult to show accuracy and reliability of this replacement.
Further, the resulting calculation of the acquisition function is extremely complicated.
On the other hand, \citet{Belakaria2019-Max} proposed to use entropy of the max-values of each dimension $\ell = 1, \ldots, L$, called \emph{max-value entropy search for multi-objective optimization} (MESMO) .
This drastically simplifies the calculation, but obviously, Pareto-frontier is not constructed only by the max-value of each axis, and thus, $\*f_{\*x} \in \cF^*$ which does not have any max-values is not preferred by this criterion. 
For other related work, we discuss in Section~\ref{sec:related-work}.

We propose another entropy-based Bayesian MOO called \emph{Pareto-frontier entropy search} (PFES).
We consider the entropy of the Pareto-frontier $\cF^*$, called Pareto-frontier entropy, defined in the space of the objective functions $\*f_{\*x}$, unlike PESMO considering the entropy of the Pareto-optimal $\*x$.
The intuition behind Pareto-frontier entropy is shown in Fig.~\ref{fig:pareto-entropy}.
By inheriting the advantage of the entropy-based approach, PFES provides a measure of global utility without any trade-off parameter.
Under a few common conditions in entropy-based methods, we show that Pareto-frontier entropy can be derived analytically by partitioning the dominated space into a set of hyper-rectangle cells.
Therefore, compared with PESMO, PFES is easy to obtain reliable evaluation of the entropy with much simpler computations.
On the other hand, PFES can capture the trade-off relation in the Pareto-frontier which is ignored by MESMO, but is essential for the MOO problem.
Although a na{\"i}ve cell partitioning can generate a large number of cells for $L \geq 3$, we show an efficient computation by using a partitioning technique used in \emph{Pareto hypervolume} computation.
This paper also discusses the \emph{decoupled setting}, which was first introduced by \citep{Hernandez-lobatoa2016-Predictive} as an extension of PESMO.
In this scenario, we assume that each one of objective functions can be observed individually. 
Since observing all the objective functions simultaneously can cause huge cost, the decoupled setting evaluates only one of objectives at every iteration.
Although this setting has not been widely studied, it is highly important in practice.
In particular, for search problems in scientific fields such as materials science and bio-engineering, multiple properties of objects (e.g., crystals, compounds, and proteins) can often be investigated separately by performing different experimental measurements or simulation-computations. 
%
% For example, in materials science, multiple physical properties are often simultaneously required to optimize for industrial purpose.
In the battery material example, conductivity and stability 
% are required to optimize for industrial purpose, and they 
can be evaluated through two independent physical simulations. 
%
% In the case of battery materials, conductivity and stability can be evaluated through two different simulation-computations. % , and usually, conductivity is more expensive to be measured.
%
Other directions of examples are also suggested by \citep{Hernandez-lobatoa2016-Predictive}, such as in robotics and design of a low calorie cookie \citep{Solnik2017-Bayesian}.
%
% In opt-genetics, proteins are designed by mutating DNA sequences, and then, biological functions are evaluated by different multiple experiments such as investigating ability of ion-transportation and measuring wavelength of light absorption.
% functional proteins with longer wavelength light absorption 
%
% Since those measurements are often quite expensive, always observing all the objective functions simultaneously can cause huge cost.
%
% Since observing a sample is quite expensive in these contexts, querying every objective every time can be large waste of sampling cost.
%
However, the existing PESMO-based acquisition function, derived by na{\"i}vely decomposing the original acquisition function, was not fully justified in a sense of the entropy.
We show that our PFES can be simply extended to the decoupled setting by considering the entropy of the marginal density without introducing any additional approximation, and it can also be obtained analytically.

Our numerical experiments show effectiveness of PFES through synthetic functions and real-world datasets from materials science.

\section{Preliminary}
\label{sec:}

We consider the \emph{multi-objective optimization} (MOO) problem which maximizes $L \geq 2$ objective functions $f^l: \cX \rightarrow \RR$ for $l = 1, \ldots, L$, where $\cX \subseteq \RR^d$ is an input domain.
Let
$\*f_{\*x} \coloneqq (f^1_{\*x}, \ldots, f^L_{\*x})^\top$, 
where
$f^l_{\*x} \coloneqq f^l(\*x)$.
The optimal solution of MOO is usually defined by \emph{Pareto optimality}.
% The goal of the multi-objective problem is to identify the set of $\*x$ which creates the \emph{Pareto frontier} of $\*f_{\*x}$.
%
For a pair of $\*f_{\*x}$ and $\*f_{\*x'}$, if
$f^l_{\*x} \geq f^l_{\*x'}$ for $\forall l \in \{ 1, \ldots, L \}$, 
we say ``$\*f_{\*x}$ dominates $\*f_{\*x'}$'' and 
the relation is denoted as
$\*f_{\*x} \succeq \*f_{\*x'}$.
If $\*f_{\*x}$ is not dominated by any other $\*f_{\*x'}$ in the domain, $\*f_{\*x}$ is called Pareto-optimal.
\emph{Pareto-frontier} $\cF^*$ is a set of Pareto-optimal $\*f_{\*x}$ which is written as
% $\cF^* \coloneqq \{ \*f_{\*x} \in \cF_{\cX} \mid \*f_{\*x'} \not\succeq \*f_{\*x}, \forall \*f_{\*x'} \in \cF_{\cX}, \*x \neq \*x' \}$, 
$\cF^* \coloneqq \{ \*f_{\*x} \in \cF_{\cX} \mid \*f_{\*x'} \not\succeq \*f_{\*x}, \forall \*f_{\*x'} \in \cF_{\cX}, \*f_{\*x} \neq \*f_{\*x'} \}$, 
where $\cF_{\cX} \coloneqq \{ \*f_{\*x} \in \RR^L \mid \forall \*x \in \cX \}$.
% where $\cF_{\cX}$ is a space created by $\*f_{\*x}$ for $\forall \*x \in \cX$.
% $\cF^* \coloneqq \{ \*f_{\*x} \mid \*f_{\*x'} \not\succeq \*f_{\*x}, \*x \neq \*x', \*x \in \cX, \*x' \in \cX \}$.
% $\cF^* \coloneqq \{ \*f_{\*x} \mid \*f_{\*x'} \succeq \*f_{\*x}, \*x \in \cX, \nexists \*x' \in \cX, \*x \neq \*x' \}$.
%
Although the Pareto-optimal points can be infinite, most strategies aim at finding a small subset of them which approximate the true $\cF^*$ with sufficient accuracy.

% -------------------------
% \subsection{Gaussian Process Regression}
% \label{ssec:GPR}

Following the standard formulation of \emph{Bayesian optimization} (BO), we model the objective function by Gaussian process regression (GPR).
%
% The objective functions are modeled by Gaussian process regression (GPR).
%
An observation for the $l$-th objective value of $\*x_i$ is assumed to be $y_{i}^l = f^l_{\*x_i} + \veps$, where $\veps \sim \cN(0,\sigma^2_{\rm noise})$.
The training dataset is written as $\cD \coloneqq \{ (\*x_i, \*y_i )\}_{i = 1}^n$, where $\*y_i = (y_{i}^1,\ldots,y_{i}^L)^\top$.
Independent $L$ GPRs are applied to each dimension with a kernel function $k(\*x,\*x')$.
By setting prior mean as $0$, the predictive mean and variance of the $l$-th GPR are 
%\[ 
% \begin{align*} 
 $\mu_l(\*x) 
% &
 =
 \*k(\*x)^\top 
 \left(
 \*K + \sigma^2_{\rm noise} \*I
 \right)^{-1}
 \*y^{l}$, 
 and
 % \text{ and } % 
 % \\
 $\sigma^2_l(\*x)
 % &
 = k(\*x,\*x) - 
 \*k(\*x)^\top 
 \left(
 \*K + \sigma^2_{\rm noise} \*I
 \right)^{-1}
 \*k(\*x)$,
% \end{align*}
% \]
where
$\*k(\*x) \coloneqq (k(\*x,\*x_1),\ldots,k(\*x_,\*x_n))^\top$,
$\*y^l \coloneqq (y_{1}^l, \ldots, y_{n}^l)^\top$, 
and $\*K$ is the kernel matrix in which the $i,j$-element is defined by $k(\*x_i,\*x_j)$.
We also define
$\*\mu(\*x) \coloneqq (\mu_1(\*x),\ldots,\mu_L(\*x))^\top$
and
$\*\sigma(\*x) \coloneqq (\sigma_1(\*x),\ldots,\sigma_L(\*x))^\top$.

% --------------------------------------------------
\section{Pareto-Frontier Entropy Search for Multi-Objective Optimization}
\label{sec:PFES}

% ToWrite: Derivatives, Low-dimensionality of $\*y$, correlated

We propose a novel information-theoretic approach to multi-objective BO (MBO).
Our method, called \emph{Pareto-frontier entropy search} (PFES), considers maximizing the information gain for the Pareto-frontier $\cF^*$.
With a slight abuse of notation, we write $\*f \preceq \cF^*$ when $\*f \in \RR^L$ is dominated by at least one of vectors in the Pareto-frontier $\cF^*$.
We define the mutual information between $\*f_{\*x}$ and the Pareto-frontier $\cF^*$ as follows:
\begin{align}
 \begin{split}  
  & I(\cF^* ; \*f_{\*x} \mid \cD) 
   \\
  & \coloneqq H[ p(\*f_{\*x} \mid \cD) ]
  -
  \EE_{\cF^*}
  \left[
  H[ p(\*f_{\*x} \mid \cD, \*f_{\*x} \preceq \cF^*) ] 
  \right],
 \end{split}
 \label{eq:mutual-info}
\end{align}
% The mutual information between $\*f_{\*x}$ and Pareto-frontier $\cF^*$ is defined as
% \begin{align}
%  I(\cF^* ; \*f_{\*x} \mid \cD) 
%  &=
%  H[ p(\*f_{\*x} \mid \cD) ]
%  -
%  \EE_{\cF^*}
%  \left[
%  H[ p(\*f_{\*x} \mid \cD, \cF^*) ] 
%  \right],
%  \label{eq:mutual-info}
% \end{align}
where $H[\cdot]$ is the differential entropy. 
%
% In Section~\ref{ssec:acquisition}, we derive an efficient computation of this mutual information inspired by a recent entropy-based Bayesian optimization called \emph{max-value entropy search} (MES) \citep{Wang2017-Max}.
%
% \red{In particular, computations can be greatly facilitated for the independent case in which most of computations can be performed analytically except for the expectation over $\cF^*$.}
%
% In Section~\ref{ssec:decoupled}, we also show that our acquisition function can be easily applied to the \emph{decoupled setting} \citep{Hernandez-lobatoa2016-Predictive} in which one of each objectives can be separately observed.
% 
% Hereafter, we refer to the case that all the objectives are observed simultaneously as the coupled setting.
% \red{Further, Section~\ref{ssec:correlated} considers the correlated objective case in which we introduce \emph{moment matching} based efficient approximation.}
% For the correlated case, we introduce \emph{moment matching} based efficient approximation.
%
In the second term, we regard the conditional distribution 
$\*f_{\*x}$ given $\cF^*$ 
as 
$p(\*f_{\*x} \mid \cD, \*f_{\*x} \preceq \cF^*)$, 
i.e., conditioning only on the given $\*x$ rather than requiring 
$\*f_{\*x} \preceq \cF^*$ for $\forall \*x \in \cX$.
Note that the same simplification has been employed by most of existing state-of-the-art information-theoretic BO algorithms including well-known \emph{predictive entropy search} (PES) and \emph{max-value entropy search} (MES) proposed by \citet{Hernandez2014-Predictive} and \citet{Wang2017-Max}, respectively.
Since the superior performance of these methods compared with other approaches has been shown, we also employ this simplification.
%
% The first term in \eqref{eq:mutual-info} is the simple $L$-dimensional Gaussian entropy, it can be analytically calculated. 
% this greatly facilitates computations
% Since this greatly facilitates computations and the superior performance compared with other approaches has been shown, we also employ this simplification.
%
% An intuition behind this mutual information is illustrated in \figurename~\ref{fig:pareto-entropy}.
%
An intuition behind this mutual information can be simply illustrated as shown in \figurename~\ref{fig:pareto-entropy}.
If the predictive distribution of the candidate $\*x$ is strongly dependent with Pareto-frontier, observing this candidate is effective to identifying Pareto-frontier.

% -------------------------
% \subsection{Acquisition Function}
\subsection{Analytical Representation of Acquisition Function}
\label{ssec:acquisition}

Since the first term in \eqref{eq:mutual-info} is the simple $L$-dimensional Gaussian entropy, it can be analytically calculated. 
% for both of the independent and correlated cases.
%
For the expectation over $\cF^*$ in the second term, we use the Monte Carlo estimation by following the standard approach in the entropy-based BO.
%
%
% With a slight abuse of notation, we write $\*f \preceq \cF^*$ when $\*f \in \RR^L$ is dominated by any one of vectors in Pareto-frontier $\cF^*$.
%
% Following the existing information-based BO \citep{Hernandez2014-Predictive,Wang2017-Max},
% We regard the conditional distribution given $\cF^*$ as $p(\*f_{\*x} \mid \cD, \*f_{\*x} \preceq \cF^*)$, i.e., conditioning only on the given $\*x$ rather than requiring $\*f_{\*x} \preceq \cF^*$ for $\forall \*x \in \cX$.
%
% Note that the same simplification has been employed by existing state-of-the-art information theoretic BO algorithms \citep{Hernandez2014-Predictive,Wang2017-Max}, and superior performance compared with other approaches has been shown.
% We set the conditional distribution given $\cF^*$ as $p(\*f_{\*x} \mid \cD, \*f_{\*x} \preceq \cF^*)$, \red{i.e., conditioning only on the given $\*x$ rather than requiring $f^{(M)}_{\*x} \leq f_*$ for $\forall \*x \in \cX$}.
%
% Note that existing state-of-the-art information theoretic BO algorithms employ the same simplification \citep{Hernandez2014-Predictive,Wang2017-Max}.
%
% Given $\cF^*$, the distribution of $\*f_{\*x}$ needs to satisfy $\*f_{\*x} \preceq \cF^*$, and 
%
%
By sampling Pareto-frontier $\cF^*$ from the current GPR model, our acquisition function is written as follows:
\begin{align}
 \begin{split}  
 a(\*x) = &
 H[ p(\*f_{\*x} \mid \cD) ] \\
 &
 -
 % \EE_{\cF^*}	\left[	H[ p(\*f_{\*x} \mid \cD, \cF^*) ] 	\right]
 % \approx
 % H[ p(\*f_{\*x} \mid \cD) ]
 %-
 % \EE_{\cF^*} 	\left[
 \frac{
 1
 }{|\mathrm{PF}|}
 \sum_{ \cF^* \in \mathrm{PF}} H[ p(\*f_{\*x} \mid \cD, \*f_{\*x} \preceq \cF^*) ],
 % \right],   
 \end{split}                    	
 \label{eq:approx-expected-entropy}
\end{align}    
where $\mathrm{PF}$ is a set of sampled Pareto-frontier $\cF^*$.
We discuss the detail of the sampling procedure in Section~\ref{ssec:sampling}.
The entropy in the second term in \eq{eq:approx-expected-entropy} is still highly complicated at a glance, because the condition $\*f_{\*x} \preceq \cF^*$ makes the distribution non-Gaussian.
We show that this entropy can be analytically represented by using a hyper-rectangle based partitioning of the dominated region.

% --------------------------------------------------
% Schematic illustration
% --------------------------------------------------
\setlength{\textfloatsep}{10pt}
\begin{figure*}[t]
 \centering
 % \igr{1}{../working_figs/truncation.pdf}
 \igr{1}{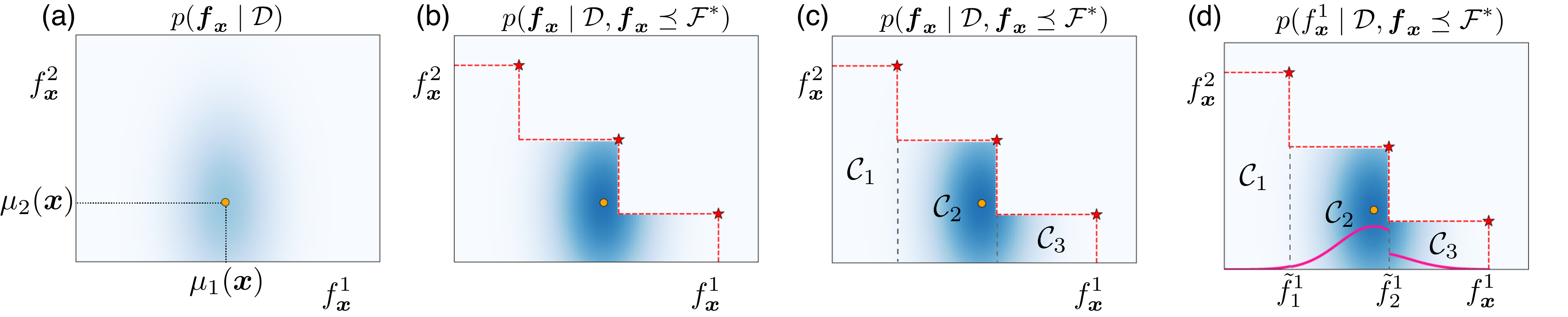}
 \caption{
 A schematic illustration of truncation by Pareto-frontier.
 (a) Original predictive distribution of two GPRs in the output space.
 (b) Predictive distribution truncated by Pareto-frontier, which results in PFTN.
 All $\*f_{\*x}$ should be dominated by the given Pareto-frontier (red stars).
 (c) Rectangle-based partitioning for the entropy evaluation.
 The entropy of PFTN is evaluated by decomposing the dominated region into rectangles called \emph{cells} ($\cC_1, \cC_2$, and $\cC_3$ in the plot).
 (d) Marginal density $p(f^1_{\*x} \mid \cD, \*f_{\*x} \preceq \cF^*)$ considered in the decoupled setting (the solid pink line).
 }
 \label{fig:PFTN}
\end{figure*}

Since the condition 
$\*f_{\*x} \preceq \cF^*$
indicates that $\*f_{\*x}$ must be dominated by Pareto-frontier $\cF^*$, the density 
$p(\*f_{\*x} \mid \cD, \*f_{\*x} \preceq \cF^*)$
is defined as a truncated distribution of the unconditional 
$p(\*f_{\*x} \mid \cD)$ 
which is the predictive distribution of GPR, i.e., the independent multi-variate normal distribution.
We call this distribution \emph{Pareto-frontier truncated normal distribution} (PFTN).
% The density 
% $p(\*f_{\*x} \mid \cD, \*f_{\*x} \preceq \cF^*)$
% in the second term is the multi-variate normal distribution truncated by Pareto-frontier, which we call \emph{Pareto-frontier truncated normal distribution} (PFTN). % as a generalized form of the truncated normal distribution. 
%
\figurename~\ref{fig:PFTN} (a) and (b) illustrate the densities before and after the truncation, respectively.
The density of PFTN
$p(\*f_{\*x} \mid \cD, \*f_{\*x} \preceq \cF^*)$
is written as 
\[ 
 p(\*f_{\*x} \mid \cD, \*f_{\*x} \preceq \cF^*) 
 =
 \begin{cases}
  \frac{1}{Z}
  p(\*f_{\*x} \mid \cD) & \text{ if } \*f_{\*x} \preceq \cF^*, \\
  % p(\*f_{\*x} \mid \cD) & \text{ if } \*f_{\*x} \in \cF, \\
  0 & \text{ otherwise }, 
 \end{cases}
\]
where $Z \in \RR$ is a normalization constant.
%
% We consider calculation of the entropy $H[ p(\*f_{\*x} \mid \cD, \*f_{\*x} \preceq \cF^*) ]$ by partitioning the dominated region $\cF \coloneqq \{ \*f \mid \*f \preceq \cF^* \}$.
%
Let $\cF \coloneqq \{ \*f \in \RR^L \mid \*f \preceq \cF^* \}$ be the dominated region, and $M \in \NN$ be the number of hyper-rectangles, called \emph{cells}, by which the region $\cF$ can be disjointly constructed as illustrated by \figurename~\ref{fig:PFTN}~(c).
In other words, we can write
$\cF = \cC_1 \cup \cC_2 \cup \ldots \cup \cC_M$, 
where the $m$-th cell $\cC_m$ is defined by 
$(\ell_m^1,u_m^{1}] \times (\ell_m^2,u_m^{2}] \times \ldots \times (\ell_m^L,u_m^{L}]$.
Note that this partitioning is created from $\cF^*$, which is generated by the current GPR model (not from the actual observed data $\{ \*y_i \}_{i=1}^n$).
% \red{$(\ell_1^m,u_1^{m}] \times (\ell_2^m,u_2^{m}] \times \ldots \times (\ell_L^m,u_L^{m}]$.}
% $(f_1^m,f_1^{m+1}] \times (f_2^m,f_2^{m+1}] \times \ldots \times (f_L^m,f_L^{m+1}]$.
%
% **************************************************
% rm start
% **************************************************
% Then, 
% the normalization constant $Z$ is written as
% % $Z$ is written as
% \begin{align}
%  Z \coloneqq 
%  \int_{\cF} p(\*f_{\*x} \mid \cD) {\rm d} \*f_{\*x}
%  =
%  \sum_{m = 1}^M
%  \int_{\cC_m} p(\*f_{\*x} \mid \cD)
%  {\rm d} \*f_{\*x},
%  \label{eq:Z}  	
% \end{align}
% which is a sum of the Gaussian integrals in the cells.
% %
% The entropy of PFTN is also decomposed into 
% \begin{align}
%  &
%  H[ p(\*f_{\*x} \mid \cD, \*f_{\*x} \preceq \cF^*) ] 
%  \nonumber \\
%  &=
%  - 
%  \int_{\cF}
%  \frac
%  {p(\*f_{\*x} \mid \cD) }
%  {Z}
%  \log 
%  \frac{p(\*f_{\*x} \mid \cD) }{Z}
%  {\rm d} \*f_{\*x}
%  \nonumber \\
%  & =
%  - 
%  \frac{1}	
%  {Z}
%  \sum_{m = 1}^M
%  \int_{\cC_m}
%  p(\*f_{\*x} \mid \cD)	
%  \log
%  p(\*f_{\*x} \mid \cD)	
%  {\rm d} \*f_{\*x}
%  +
%  \log {Z}.
%  \label{eq:partitioned-entropy}
% \end{align}
% **************************************************
% rm end
% **************************************************

Let 
% $\alpha_{m,l} \coloneqq (f_l^m - \mu_l(\*x)) / \sigma_l(\*x)$,
$\ubar{\alpha}_{m,l} \coloneqq (\ell^l_m - \mu_l(\*x)) / \sigma_l(\*x)$, 
$\bar{\alpha}_{m,l} \coloneqq (u^l_m - \mu_l(\*x)) / \sigma_l(\*x)$, 
$Z_{ml} \coloneqq \Phi(\bar{\alpha}_{m,l}) - \Phi(\ubar{\alpha}_{m,l})$, 
and 
$Z_m \coloneqq \prod_{l=1}^L Z_{ml}$,
where $\Phi$ is the standard Gaussian cumulative distribution function (CDF).
For the entropy in the second term of \eq{eq:approx-expected-entropy},
% \eqref{eq:partitioned-entropy}, 
% we 
the cell-based decomposition of the dominated region derives the following theorem (the proof is in Appendix~\ref{sec:proof-theorem1}):
\begin{theo}
 \label{thm:entropy-reduction}
 For $L$ independent GPRs, the entropy of PFTN $p(\*f_{\*x} \mid \cD, \*f_{\*x} \preceq \cF^*)$ is given by
%  \[
 \begin{align}  
  \begin{split}   
  & 
  H[ p(\*f_{\*x} \mid \cD, \*f_{\*x} \preceq \cF^*) ]
  \\
  & 
  =
  \log
  \left(
  (\sqrt{2 \pi e})^L
  Z
  \prod_{l=1}^L \sigma_l(\*x) 
  \right)
  +
  % Z^{-1}
  \sum_{m = 1}^M
  % Z_{m}
  \frac{Z_{m}}{Z}  \sum_{l = 1}^L
  % \prod_{l' \neq l} Z_{m l'}
  \Gamma_{ml},
  % \frac 	 {\alpha_{m,l} \phi(\alpha_{m,l}) - \alpha_{m+1,l} \phi(\alpha_{m+1,l})} 	 {2 Z_{ml}},
  \end{split}
  \label{eq:entropy-given-frontier}
 \end{align}
%\]
 where 	
 $\Gamma_{ml} \coloneqq ({\ubar{\alpha}_{m,l} \phi(\ubar{\alpha}_{m,l}) - \bar{\alpha}_{m,l} \phi(\bar{\alpha}_{m,l})}) / ({2 Z_{ml}})$ with standard Gaussian probability density function (PDF) $\phi$, and
\begin{align}
 Z 
 &= 
 % \sum_{m = 1}^M
 % \int_{\cC_m} p(\*f_{\*x} \mid \cD)
 % {\rm d} \*f_{\*x}
 % =
 \sum_{m = 1}^M
 \prod_{l = 1}^L 
 % \int_{f_l^m}^{f_l^{m+1}}
 \int_{\ell^l_m}^{u^l_{m}}
 p(f^l_{\*x} \mid \cD)
 {\rm d} f^l_{\*x}	
 =
 \sum_{m = 1}^M
 Z_m.
 % \prod_{l = 1}^L Z_{ml}.
 \label{eq:Z-decomposed}
\end{align}
\end{theo}
\noindent
This entropy is a simple function of the predictive distribution of GPR at $\*x$ and the Gaussian PDF/CDF functions, and thus, we can easily evaluate \eq{eq:entropy-given-frontier} if the cell-based partitioning is available.
For the partitioning, we discuss in Section~\ref{ssec:partitioning}.
% For more detail of computations, we discuss in Section~\ref{sec:computation}.

% --------------------------------------------------
\subsection{Computation of Pareto-frontier Entropy}

Suppose that we already have the predictive distribution of $\*x$, i.e., $\*\mu(\*x)$ and $\*\sigma(\*x)$, a set of sampled Pareto-frontier ${\rm PF}$, and a set of cells $\{ \cC_m \}_{m=1}^M$.
Then, the normalization constant $Z$ \eq{eq:Z-decomposed} is calculated by $O(ML)$, and the acquisition function \eq{eq:approx-expected-entropy} 
% of the coupled setting
can also be obtained by $O(ML)$.
We here describe the sampling procedure of $\cF^*$, and the cell partitioning of the dominated region.

% --------------------------------------------------
\subsubsection{Sampling Pareto-frontier}
\label{ssec:sampling}

PFES first needs to sample a set of Pareto-frontier $\cF^*$.
%  by generating objective functions from the posterior of GPR.
%
% In the case of a discrete candidate space $\cX$ (i.e., so-called the pooled setting), generating function values for all $\cX$ needs $O(|\cX|^3)$ computations, and thus, for a large candidate set, approximations such as \emph{random feature map} (RFM) \citep{Rahimi2008-Random} are effective.
%
% Sampling from RFM needs cubic computations with respect to the dimension of random features which is typically set by less than $1000$.
% 
% When the function values are obtained, Pareto-frontier of the generated function values can be identified by $O(|\cX| \log |\cX|)$ for $L = 2$ (with divide/conquer algorithm), or $O(L |\cX|^2)$ for general $L > 2$.
%
% If $|\cX|$ is quite large to identify the Pareto set by using those direct approaches, general MOO algorithms such as NSGA-II \citep{Deb2002-Fast} is effective heuristics.
%
For this step, we follow an approach proposed by the existing information-based MBO \citep{Hernandez-lobatoa2016-Predictive}.
They employed \emph{random feature map} (RFM) \citep{Rahimi2008-Random} to approximate of the current GPR by a Bayesian linear model $\*w_{l}^\top \phi(\*x)$, where $\*w_{l} \in \RR^D$ is a parameter for the $l$-th objective and $\phi: \cX \rightarrow \RR^D$ is a pre-defined basis vector.
By generating $\*w_{l}$ from the posterior, we can sample a ``function'' 
$\*w_{l}^\top \phi(\*x)$
with the cubic computational cost with respect to $D$, typically less than $1000$.
A sample of $\cF^*$ can be obtained through solving MOO 
on
$\*w_{l}^\top \phi(\*x)$ for $l = 1, \ldots, L$.
Since the objective is explicitly known as $\*w_{l}^\top \phi(\*x)$, general MOO algorithms such as NSGA-II \citep{Deb2002-Fast} is applicable to solving this MOO.
%
% In our experiments, we sampled 10 Pareto-frontier by this procedure.
%
It has been shown empirically that entropy-based approaches are robust with respect to the number of this sampling \citep[e.g.,][]{Wang2017-Max}, and usually only the small number of samples are used (e.g., $10$).
In the later experiments, we evaluate sensitivity of performance to this setting.
%
% \blue{For a discrete candidate space $\cX$ (i.e., so-called the pooled setting), we discuss in Appendix~\ref{app:}.}

% --------------------------------------------------
\subsubsection{Partitioning of Dominated Region}
\label{ssec:partitioning}

For the generated Pareto-frontier, we need to construct a set of cells $\{ \cC_m \}_{m =  1}^M$.
A similar cell-based decomposition has been performed by existing MBO methods such as the well-known expected improvement-based method \citep{Shah2016-Pareto} in a slightly different context.
Although \citet{Shah2016-Pareto} employed a na{\"i}ve grid-based partitioning, which produces $O(|\cF^*|^L)$ cells, this may cause large computational cost, particularly when $L > 2$.
We show a method for the partitioning by which the number of cells can be drastically reduced compared with the na{\"i}ve approach.

For $L = 2$, which is the most common setting in MOO, there exists a decomposition with $M = |\cF^*|$ as we can clearly see in \figurename~\ref{fig:PFTN}~(c).
Most of MOO algorithms (such as NSGA-II, used for generating $\cF^*$) can explicitly specify the maximum number of $|\cF^*|$ beforehand.
This value is typically set as at most a few hundreds \citep{Deb2014-Evolutionary} even for a large $L$ more than $10$, which does not commonly occur in real-world MOO problems.
We also empirically observed that a small $|\cF^*|$ is sufficient to capture the trade-off relation among objectives.
In our later experiments, we set $|\cF^*| = 50$ by following an existing information-based approach \citep{Hernandez-lobatoa2016-Predictive}, based on which PFES showed stable performance.
%
% In this case, the number of the cells are at most $50$.

For $L > 2$, the simple partitioning like \figurename~\ref{fig:PFTN}~(c) is not applicable.
To produce a smaller number of cells, we propose to use techniques in the \emph{Pareto hyper-volume} computation.
Pareto hyper-volume is defined by the volume of the region dominated by Pareto-frontier, which is widely used as an evaluation measure of MOO.
Therefore, many studies have been devoted to its efficient computation mainly by decomposing the region into as few cells as possible \citep{Couckuyt2014}.
Through this decomposition, we can obtain a partitioning as a by-product of the hyper-volume computation.
%
% For example, WFG \citep{Couckuyt2014} and 
For example, quick hyper-volume (QHV) \citep{Russo2014-Quick} is one of well-known methods which recursively calculates the volume by partitioning the region with a quick-sort like divide-and-conquer procedure.
%
% $O(d n^{1+\epsilon} \log^{d-2} n)$
Under a few assumptions on the distribution of $\cF^*$, QHV takes 
$O(L |\cF^*|^{1+\epsilon} \log^{L-2} |\cF^*|)$
time for the average case with high probability, where $\epsilon > 0$ is an arbitrary small constant \citep[see][for the detail]{Russo2014-Quick}.
Since the hyper-volume computation is still actively studied, more advanced algorithms are also applicable if it produces rectangle regions as a by-product of the algorithm.
In this paper, we employ QHV because of its efficiency and simplicity.   
Therefore, we can obtain the small number of cells, by which the efficient evaluation of the entropy becomes possible.
%  (Appendix~\ref{app:time} shows empirical evaluation of computational cost)
%
% \blue{
% For the decoupling case, the partitioning shown in \figurename~\ref{fig:cells3d} can be created from the QHV partitioning.
% %
% For each interval $(\tilde{f}_{s}^{l},\tilde{f}_{s+1}^{l}]$, if a cell $\cC_m$ contains the interval, we extract a hyper-rectangle $\cC \subseteq \cC_m$ in which only the interval of the $l$-the dimension of $\cC_m$ is replaced with $(\tilde{f}_{s}^{l},\tilde{f}_{s+1}^{l}]$.
% %
% This procedure increase the total number of cells at most $|\cF^*|$ which we usually set a small value ($50$ in this paper).
% }

% --------------------------------------------------
\section{Extension to Decoupled Setting}
\label{ssec:decoupled}

In the previous section, we assumed that all the objectives are observed simultaneously, to which we refer as the \emph{coupled setting}.
In contrast, the \emph{decoupled setting} assumes that each one of objectives can be separately observed.
%
%, and in each iteration we need to determine both of $\*x$ and an index of objective $l$ to be observed.
%
% \citet{Hernandez-lobatoa2016-Predictive} indicated that this is particularly useful to identify difficult objectives that require more evaluations. % , and also effective when the observation cost of each objective is different.
%
Although this setting has not been widely studied in the MBO literature, this can be a significant problem setting particularly in the case that the sampling cost of each objective is highly expensive, because then, observing all the objectives every time may cause large amount of waste-of-cost.
Further, we also focus on the fact that the cost of observations are often different among multiple objective functions. 
Therefore, we introduce a \emph{cost-sensitive} acquisition function into the decoupled setting.

In this setting, we need to determine a pair of an input $\*x$ and an index of an objective $l$ to be observed.
PFES can provide a natural criterion for this purpose by considering the mutual information between $\cF^*$ and the $l$-th objective $I(\cF^* ; f^l_{\*x})$.
Then, we can define the following cost-sensitive acquisition function:
% with the same sampling strategy of $\cF^*$ as the previous section: 
\begin{align}
 \begin{split}  
 a(\*x, l) =
 & 
 \frac{1}{\lambda_l} 
 % \left\{
 \Bigl\{
 H[p(f^l_{\*x} \mid \cD)]
 \\
 &
 -
 \frac{1}{|\mathrm{PF}|}
 \sum_{\cF^* \in \mathrm{PF}}
 H[p(f^l_{\*x} \mid \cD, \*f_{\*x} \preceq \cF^*)]
 \Bigr\}
 % \right\},
 \end{split}
 \label{eq:acq-single}
\end{align}
where $\lambda_l > 0$ is the observation cost of the $l$-th objective which is assumed to be known beforehand.
A pair of $\*x$ and $l$ to be queried can be determined by $\argmax_{\*x,l} a(\*x,l)$.
Here again, the first term of \eqref{eq:acq-single} is easy to calculate.
We derive an efficient computation for the entropy in the second term.
\figurename~\ref{fig:PFTN}~(d) shows an illustration of the density $p(f^l_{\*x} \mid \cD, \*f_{\*x} \preceq \cF^*)$ in the second term.

\setlength{\textfloatsep}{10pt}
\begin{figure}[t]
% \begin{wrapfigure}[13]{r}[0mm]{.33\tw}
%  \vspace{-1em} 
 \centering
 % \igr{.25}{../working_figs/cells3d.pdf}
 \igr{.2}{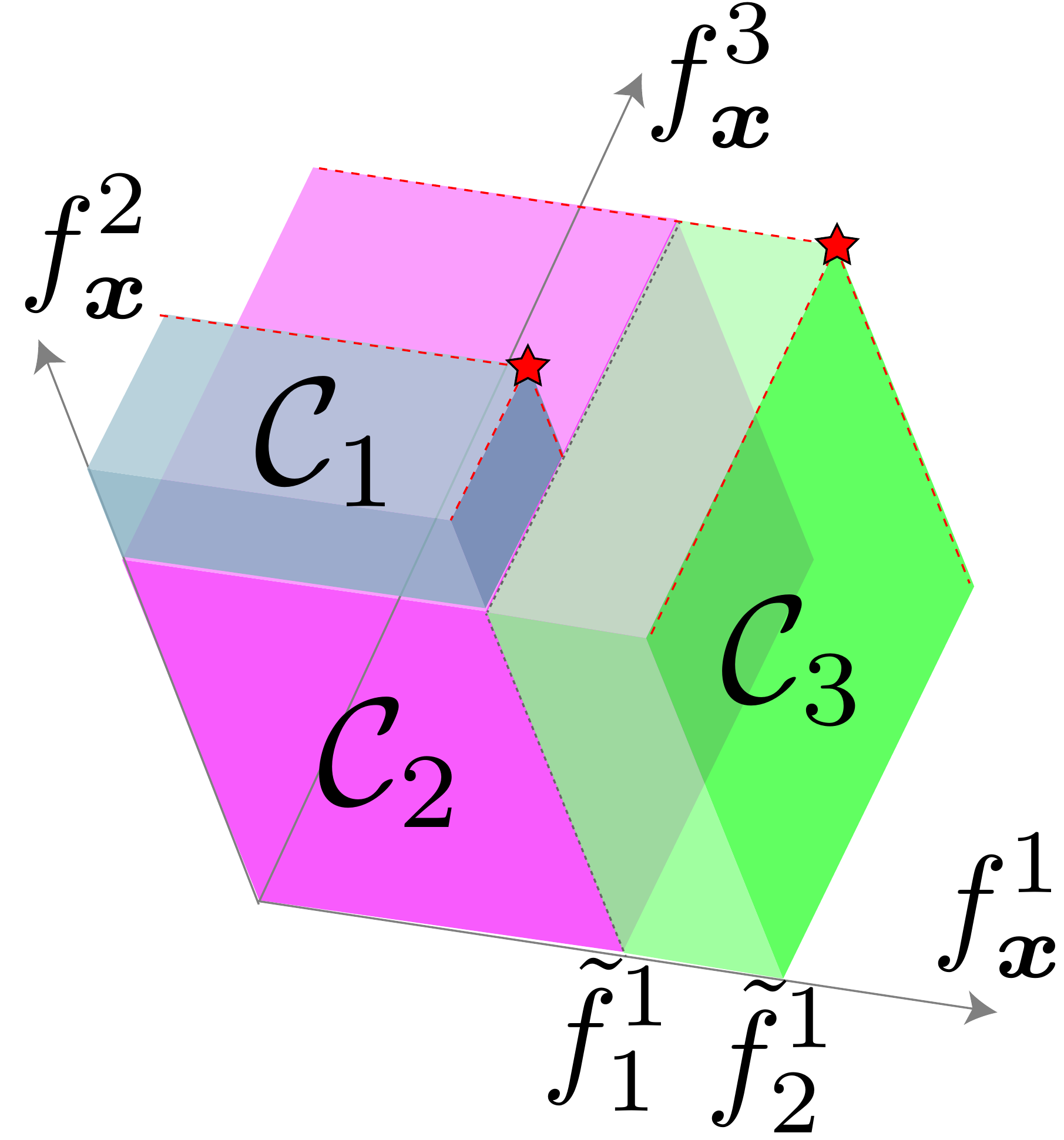}
 \caption{
 An example of partitioning for decoupled setting.
 In this case, $\cM(l,s)$ for $l = 1$ are $\cM(1,0) = \{ 1, 2 \}$, and $\cM(1,1) = \{ 3 \}$.
 }
 \label{fig:cells3d}
% \end{wrapfigure}
\end{figure}

Define $S \coloneqq |\cF^*|$ as the number of the Pareto optimal points, and $\tilde{f}_1^l, \ldots, \tilde{f}_{S_l}^{l}$ for $S_l \leq S$ as a sequence ascendingly sorted by the $l$-th dimension of $\forall \*f_{\*x} \in \cF^*$ in which duplicated values are eliminated.
%
% Let $s_l^{(f)} \in \{ 0, \ldots, S_l - 1 \}$ be the index $s$ such that $f \in (\tilde{f}^l_{s},\tilde{f}^l_{s+1}]$ where $\tilde{f}_0^l \coloneqq -\infty$.
%
% For the $l$-th dimension interval of $\cC_m$,%with $\forall m$, 
% i.e., $(f_l^m, f_l^{m+1}]$, 
For $\forall m \in \{1,\ldots,M\}$, we assume that there exists
$s \in \{0,\ldots, S_l\}$
such that
% $(f_l^m, f_l^{m+1}] = (\tilde{f}_l^{m'},\tilde{f}_l^{m'+1}]$
$(\ell^l_m, u^l_{m}] = (\tilde{f}^l_{s},\tilde{f}^l_{s+1}]$
% (this assumption is just for notational simplicity, and, without this assumption, we can derive the same computational procedures  by a slight modification
(this assumption is just for notational simplicity, and we can create the cells in such a way that this condition is satisfied).
The marginal density of PFTN
$p(f^l_{\*x} \mid \cD, \*f_{\*x} \preceq \cF^*)$
depends on the interval $(\tilde{f}^l_s, \tilde{f}^l_{s+1}]$ that $f^l_{\*x}$ exists.
Let
$% \cM_{l,m}
\cM(l,s)
 \coloneqq
 \{ m \mid 
 (\ell^l_{m}, u^l_{m}] =
 (\tilde{f}^l_s, \tilde{f}^l_{s+1}]
 % \forall \*f \in \cC_{m'},  f_l \in  (\tilde{f}_l^m, \tilde{f}_l^{m+1}]
\}$
% Let
% $\cM_{l,m}
%  \coloneqq
%  \{ m' \mid 
%  \exists
%  f_l \in
%  (\tilde{f}_l^m, \tilde{f}_l^{m+1}]
%  \text{ for } \forall \*f \in \cC_{m'} 
% \}$
be the index set of $\cC_{m}$ in which the $l$-th dimension is equal to $(\tilde{f}^l_s, \tilde{f}^l_{s+1}]$ as illustrated in \figurename~\ref{fig:cells3d}, 
and
$s_l^{(f)} \in \{ 0, \ldots, S_l - 1 \}$ be the index $s$ such that $f \in (\tilde{f}^l_{s},\tilde{f}^l_{s+1}]$, where $\tilde{f}_0^l \coloneqq -\infty$. 
%
% \figurename~\ref{fig:cells3d} shows an illustration of $\cM(l,s)$.
%
% **************************************************
% rm start
% **************************************************
% Then, the marginalization can be represented as
% \begin{align}
%  % &
%  % p(f^l_{\*x} \red{\in (\tilde{f}_l^m, \tilde{f}_l^{m+1}]} \mid \cD, \*f_{\*x} \preceq \cF^*) 
%  & 
%  p(f^l_{\*x} \mid \cD, \*f_{\*x} \preceq \cF^*) 
%  \nonumber
%  \\
%  & =
%  \sum_{m \in \cM(l,s_l^{(f^l_{\*x})})}
%  \int_{\cC_{m}^{{\setminus l}}}
%  \frac{p(\*f_{\*x} \mid \cD)}{Z}
%  {\rm d} \*f^{{\setminus l}}_{\*x} 
%  \nonumber
%  \\
%  & =
%  \frac{p(f_{\*x}^l \mid \cD)}{Z}
%  \sum_{m \in \cM(l,s_l^{(f^l_{\*x})})}
%  \int_{\cC_{m}^{{\setminus l}}}
%  p(\*f^{{\setminus l}}_{\*x} \mid f_{\*x}^l, \cD) 
%  {\rm d} \*f^{{\setminus l}}_{\*x},		
%  \label{eq:marginal-given-frontier}
% \end{align}
% where 
% $\cC_{m}^{{\setminus l}}$ is the $(L-1)$-dimensional cell created by eliminating the $l$-th dimension of $\cC_{m}$, and $\*f^{{\setminus l}}_{\*x}$ is a subvector of $\*f_{\*x}$ without the $l$-th dimension.
%
% **************************************************
% rm end
% **************************************************

Using independence of the objectives, 
% we can decompose the conditional density inside the integral of \eqref{eq:marginal-given-frontier} as
% $p(\*f_{\*x}^{{\setminus l}} \mid f_{\*x}^l, \cD) = \prod_{l' \neq l} p(f_{\*x}^{l'} \mid  \cD)$, 
% which derives the following theorem:
we derive the following theorem (the proof is in Appendix~\ref{sec:proof-theorem2}):
% --------------------------------------------------
% Independent Decoupled 
% --------------------------------------------------
\begin{theo}
 \label{thm:single-entropy-ind}
 For $L$ independent GPRs, the entropy of $p(f^l_{\*x} \mid \cD, \*f_{\*x} \preceq \cF^*)$ is given by
 \begin{align}
  & H[ p(f^l_{\*x} \mid \cD, \*f_{\*x} \preceq \cF^*) ] 
  = 
  \nonumber
  \\
  & 
  - \sum_{s = 0}^{S_l-1}
  \frac{ 
  \sum_{m \in \cM(l,s)}
  Z_{m}
  }{Z}
  \Biggl(
  \log
  \frac{ 
  \sum_{m \in \cM(l,s)}
  Z_{m}
  }{
  Z
  \sqrt{2 \pi e} \sigma_{l}(\*x) {\tilde{Z}_{sl}}
  }
  % - \log
  % (\sqrt{2 \pi e} \sigma_{l}(\*x) {\tilde{Z}_{sl}})
  -
  \tilde{\Gamma}_{sl}
  \Biggr),
  \label{eq:single-entropy-ind}
 \end{align}
 where
 $\tilde{Z}_{sl} \coloneqq \Phi(\tilde{\alpha}_{s+1,l}) - \Phi(\tilde{\alpha}_{s,l})$ with 
 $\tilde{\alpha}_{s,l} \coloneqq (\tilde{f}^{l}_s - \mu_l(\*x))/\sigma_l(\*x)$, and
 % $\tilde{\Gamma}_{ml} \coloneqq ({ \tilde{\alpha}_l^{m+1} \phi(\tilde{\alpha}_l^{m+1}) - \tilde{\alpha}_l^{m} \phi(\tilde{\alpha}_l^{m}) })/({ 2 \tilde{Z}_{ml}})$.
 $\tilde{\Gamma}_{sl} \coloneqq ({ \tilde{\alpha}_{s,l} \phi(\tilde{\alpha}_{s,l}) - \tilde{\alpha}_{s+1,l} \phi(\tilde{\alpha}_{s+1,l}) })/({ 2 \tilde{Z}_{sl}})$.
\end{theo}
\noindent
As shown in this theorem, even for the decoupled case, we obtain an analytical representation of the entropy.
Although the equation may look complicated at first sight, this can be easily calculated from the predictive distribution if the cell-based partitioning is given.
% As in the case of the coupled setting, the computation of this entropy is simple, when the partitioning is obtained.
% this can also be simply evaluated given the cells.

The computation for the acquisition function of the decoupled setting is similar to the coupled case.
We first sample $\cF^*$ from the current GPR model, and then, creating the cell partitioning for each sampled Pareto-frontier.
The partitioning shown in \figurename~\ref{fig:cells3d} can also be created from the QHV partitioning.
For each interval $(\tilde{f}_{s}^{l},\tilde{f}_{s+1}^{l}]$, if a cell $\cC$ created by QHV contains this interval, we extract a sub-cell $\cC' \subseteq \cC$ in which only the interval of the $l$-the dimension of $\cC$ is replaced with $(\tilde{f}_{s}^{l},\tilde{f}_{s+1}^{l}]$.
This procedure increases the total number of cells at most $|\cF^*|$ times which we usually set a small value ($50$ in this paper).
After the partitioning, for a given predictive distribution of GPR, the acquisition function \eq{eq:acq-single} can be simply calculated with $O(M |\cF^*| L)$ by using \eq{eq:single-entropy-ind}.

\section{Related Work}
\label{sec:related-work}

For the black-box MOO problem, the combination of scalarization and evolutionary computations have been quite popular \citep[e.g.,][]{Knowles2006-ParEGO,Zhang2010-Expensive}.
In particular, ParEGO \citep{Knowles2006-ParEGO} has been widely known for its outstanding performance.
The scalarization approach transforms MOO into a single-objective problem by which the Pareto-optimal solutions can be obtained under the certain regularity conditions. 
However, acquisition functions for the transformed single-objective are expected to be suboptimal.
Although recently, some studies \citep{Paria2018-Flexible, Marban2017-Multi} have explored extensions of scalarization for identifying a specific subset of Pareto-frontier, we only focus on identifying the entire Pareto-frontier in this paper.

Acquisition functions in usual BO have been extended to MOO.
An extension of standard \emph{expected improvement} (EI) to MOO considers EI of Pareto hyper-volume \citep{Emmerich2005-Single}, which we call \emph{expected hyper-volume improvement} (EHI).
 % , which is a volume of an area dominated by the Pareto-frontier.
%
% We call this approach \emph{expected hyper-volume} (EHI).
%
Further, \citet{Shah2016-Pareto} extended EHI to correlated objectives
(Although we only focus on the independent case in this paper, Appendix~\ref{app:correlated} shows that our PFES can be extended to the correlated case).
Although EI is a widely accepted criterion, it only measures the local utility.
\emph{Upper confidence bound} (UCB) is another well-known acquisition function for BO \citep{Srinivas2010-Gaussian}.
SMSego \citep{Ponweiser2008-Multiobjective} is one of UCB based approaches to MOO which optimistically evaluates the hyper-volume.
PAL and $\epsilon$-PAL \citep{Zuluaga2013-Active,Zuluaga2016-e-PAL} are another UCB approaches in which a confidence interval based evaluation of Pareto-frontier is proposed.
\citet{Shilton2018-Multi} evaluate the distance between a querying point and Pareto-frontier for defining a UCB criterion.
A common difficulty of UCB approach is its hyper-parameter which balances the effect of the uncertainty term.
Although there often exist theoretical suggestions for determining it, careful tuning is necessary in practice since those suggested values usually contain some unknown constant.

% There exist several other directions of uncertainty based approaches.
% Uncertainty$B$K4p$E$/%"%W%m!<%A$b$$$/$D$+$N8&5f$,$"$k!%(B
%
As another uncertainty based approach, \citet{Campigotto2014-Active} considers \emph{uncertainty sampling} for directly modeling Pareto-frontier as a function.
% \cite{Campigotto2014-Active}$B$O(Bsimple$B$J(Buncertainty sampling$B$r(BPareto$B:GE,2=$KE,MQ$7$?(B. 
%
Although the simplest uncertainty sampling only measures local uncertainty at a querying point, global uncertainty measures have also been studied.
% uncertainty sampling$B$O(Blocal$B$J(Buncertainty$B$N$_$r8+$k(Bsimple$B$JJ}K!$G$"$k$,(B, global$B$K(Buncertainty$B$rI>2A$9$k$3$H8&5f$b$$$/$D$+B8:_$9$k(B. 
%
SUR \citep{Picheny2015-Multiobjective} considers the expected decrease of \emph{probability of improvement} (PI) as a measure of uncertainty reduction.
% SUR\cite{Picheny2015-Multiobjective}$B$O(Bhyper-volume$B$N(BPI$B$K4p$E$/;XI8$G$"$k(B. 
% $B$"$k8uJdE@(B$\*x$$B$rI>2A$7$?8e$K!$(B$\*cX$$BA4BN$G(BPI$B$N8:>/NL$K$h$C$F(Buncertainty reduction$B$H$9$k(B. 
%
However, SUR is computationally extremely expensive, because the PI after a querying point is added to the training set is needed to be integrated over the entire $\cX$.
% $B$?$@$7!$(B$\*f(\*x)$$B$r%5%s%W%j%s%0$7!$$=$N8e!$6u4VA4BN$G$N(BPI$B$N@QJ,$r?tCM7W;;$9$k$?$a(Bextremely expensive$B$G$"$k(B.
%
According to \citep{Hernandez-lobatoa2016-Predictive}, SUR is feasible only $2$ or $3$ objectives.
%
% Further, this computational difficulty would also limit the dimension of $\cX$ because the numerical integration in the entire $\cX$ is required to evaluate each query.
%
Further, scalability for the dimension of the input space $\cX$ is also severely limited because of the required numerical integration in the entire $\cX$.

The effectiveness of information-based approaches have been shown for single objective BO \citep{Henning2012-Entropy,Hernandez2014-Predictive,Wang2017-Max}
.
In MOO, a seminal work on this direction is \emph{predictive entropy search for multi-objective optimization} (PESMO) \citep{Hernandez-lobatoa2016-Predictive}.
PESMO considers the entropy of a set of Pareto optimal $\*x$, called Pareto set $\cX^*$.
Similar to PFES, PESMO first samples $\cX^*$ from the current model.
However, unlike our PFES, the entropy in PESMO is not reduced to an analytical formula.
An approximation with \emph{expectation propagation} (EP) \citep{Minka2001-Expectation} was proposed, but this results in that the each dimension of the conditional density
$p(\*f_{\*x} \mid \cD, \cX^*)$ 
is approximated by the independent Gaussian distribution, whose accuracy and appropriateness are not clarified.
Further, the computational procedure of this approximation is highly complicated.
By contrast, in our PFES, PFTN $p(\*f_{\*x} \mid \cD, \*f_{\*x} \preceq \cF^*)$ and its entropy can be written analytically, and thus, the dependent relation in this density is incorporated into the acquisition function evaluation.
From \figurename~\ref{fig:PFTN} (b), we can clearly see that $p(\*f_{\*x} \mid \cD, \*f_{\*x} \preceq \cF^*)$ can have dependent relation among $\*f_{\*x}$ nevertheless the original GPR is assumed to be independent.
Although PESMO is an only method which deals with the decoupled setting, their acquisition function, which is derived by simply decomposing the information gain of the coupled setting, was not rigorously justified as the entropy.
\emph{Max-value entropy search for multi-objective optimization} (MESMO) \citep{Belakaria2019-Max} is another information-based MOO which uses the entropy of the max-values of each dimension $l = 1, \ldots, L$.
MESMO is inspired by \emph{max-value entropy search} (MES) of single-objective BO \citep{Wang2017-Max} which considers the entropy of the optimal output $\max_{\*x} f(\*x)$.
This approach drastically simplifies the calculation, but obviously in MOO, Pareto-frontier is not constructed only by the max-value of each axis, and thus, $\*f_{\*x} \in \cF^*$ which does not have any max-values is not preferred by this criterion. 
For example, although the red star point in \figurename~\ref{fig:pareto-entropy} (a) is not the maximum in the both axes, this point largely improves the Pareto-frontier created by the already observed points (white circles).
PFES is also closely related to MES in a sense that the entropy of the output space is used to measure the information gain.

\section{Experiments}
\label{sec:}

We compared PFES with ParEGO, EHI, SMSego, and MESMO.
% We compare PFES with ParEGO, EHI, SMSego, and PESMO.
%
To evaluate performance, we used the hyper-volume of the region dominated by Pareto-frontier, which is a standard evaluation measure in MOO. 
For the kernel function in all the methods, we employed the Gaussian kernel 
$k(\*x,\*x^\prime) = \exp( - \| \*x - \*x^\prime \|_2^2 / (2 \sigma^2))$.
The samplings of $\cF^*$ in PFES and $\cX^*$ in MESMO, which we call \emph{Pareto sampling}, were performed $10$ times, respectively.
For the cell partitioning of PFES, we used the QHV algorithm \citep{Russo2014-Quick}.
For the acquisition function maximization of all methods, we used the DIRECT algorithm \citep{Jones1993-Lipschitzian}.
The performance is evaluated by the hyper-volume created by already observed instances relative to the optimal hyper-volume, which we call \emph{relative hyper-volume} (RHV).
The other experimental settings are shown in Appendix~\ref{app:settings}.
Because of implementation and computational complexity issues, we could not perform comparison with PESMO and SUR while they provide the global measures of utility for MOO (the author implementation was not compatible with our environment).
% (see Appendix~\ref{app:other-methods} for detail).
We believe that the above compared methods are currently widely used, and thus, would be sufficient as the baseline to verify the performance of PFES.

% --------------------------------------------------
\subsection{Benchmark Functions}
\label{sec:benchmark}

We first used benchmark functions which have continuous domain $\cX$.
Each experiment run $10$ times with a different set of initial observations which were randomly selected $5$ points.
Here, we consider the coupled setting.
For Pareto sampling, NSGA-II was applied to functions generated from RFM with $500$ basis functions, and we set the maximum size of Pareto set as $50$ by following \citep{Hernandez-lobatoa2016-Predictive}.
The results on four MOO problems are shown in \figurename~\ref{fig:benchmark}.
In the figure, (a) Ackley/Sphere is created by combining two single objective benchmark functions $L = 2$ with $d = 2$ \citep{Surjanovic2013-Virtual}, and (b) - (d) are from well-known MOO benchmark functions \citep{Huband2006-Review}.
ZDT4 has two objectives $L = 2$ and the input dimension is $d = 4$.
DTLZ3 and 4 have four objectives $L = 4$ and the input dimension is $d = 6$.
To calculate RHV, the optimal hyper-volume is estimated by applying NSGA-II to the true objective function.
Here, in PFES, we evaluate the three settings of the number of samplings $|\mathrm{PF}| = 10, 30$, and $50$.

\figurename~\ref{fig:benchmark} shows the results.
We see that PFES (any of 10, 30 or 50) achieved the fastest convergence for Ackley/Sphere, DTLZ3, and DTLZ4, and for DTLZ3, PFES was roughly the second best among the compared methods.
The three settings of $|\mathrm{PF}|$ showed similar behaviors, suggesting that the performance dependence of PFES on $|\mathrm{PF}|$ is small.
MESMO showed the faster convergence on ZDT4, while for the two MOO benchmark functions DTLZ3 and DTLZ4, it was relatively slow.
Among the three MOO benchmark functions, only ZDT4 has the ``convex'' dominated region $\cF$, while DTLZ3 and DTLZ4 have the ``concave'' $\cF$ \citep[see][for the detail]{Huband2006-Review}.
The stronger trade-off relation exists in the concave case, for which MESMO failed to improve RHV rapidly.

\setlength{\textfloatsep}{10pt}
\begin{figure}[t]
 \centering 
 % --------------------------------------------------
 % Old results
 % --------------------------------------------------
 % \subfloat[Ackley/Sphere]{
 % \igr{0.245}{./figs/Simple_Rel_HV_Ackley_and_Sphere_2.pdf}
 % }
 % \subfloat[ZDT4]{
 % \igr{0.245}{./figs/Simple_Rel_HV_ZDT4.pdf}
 % }
 % \subfloat[DTLZ3]{
 % \igr{0.245}{./figs/Simple_Rel_HV_DTLZ3.pdf}
 % }
 % \subfloat[DTLZ4]{
 % \igr{0.245}{./figs/Simple_Rel_HV_DTLZ4.pdf}
 % }
 \subfloat[Ackley/Sphere]{
 \igr{0.24}{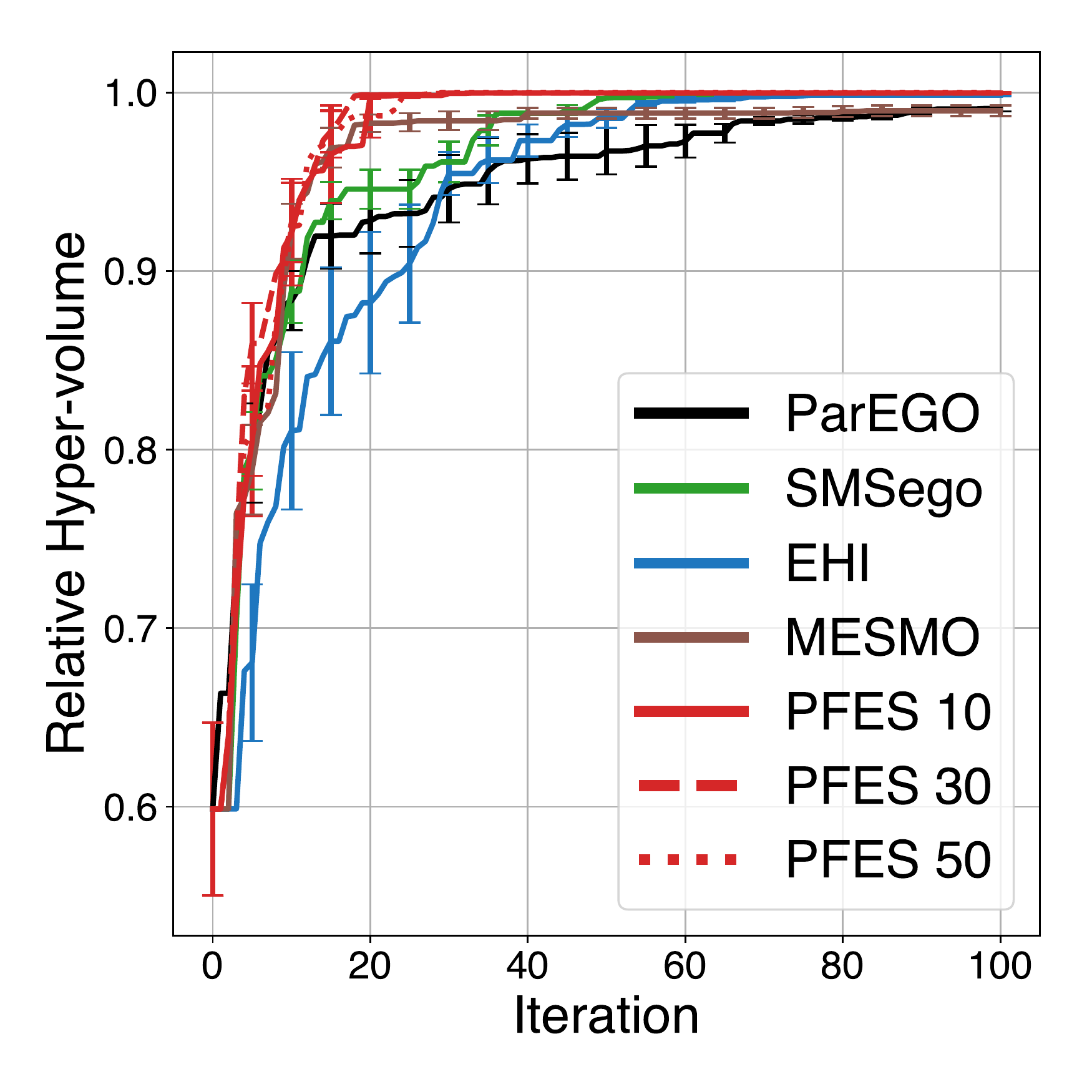}
 }
 \subfloat[ZDT4]{
 \igr{0.24}{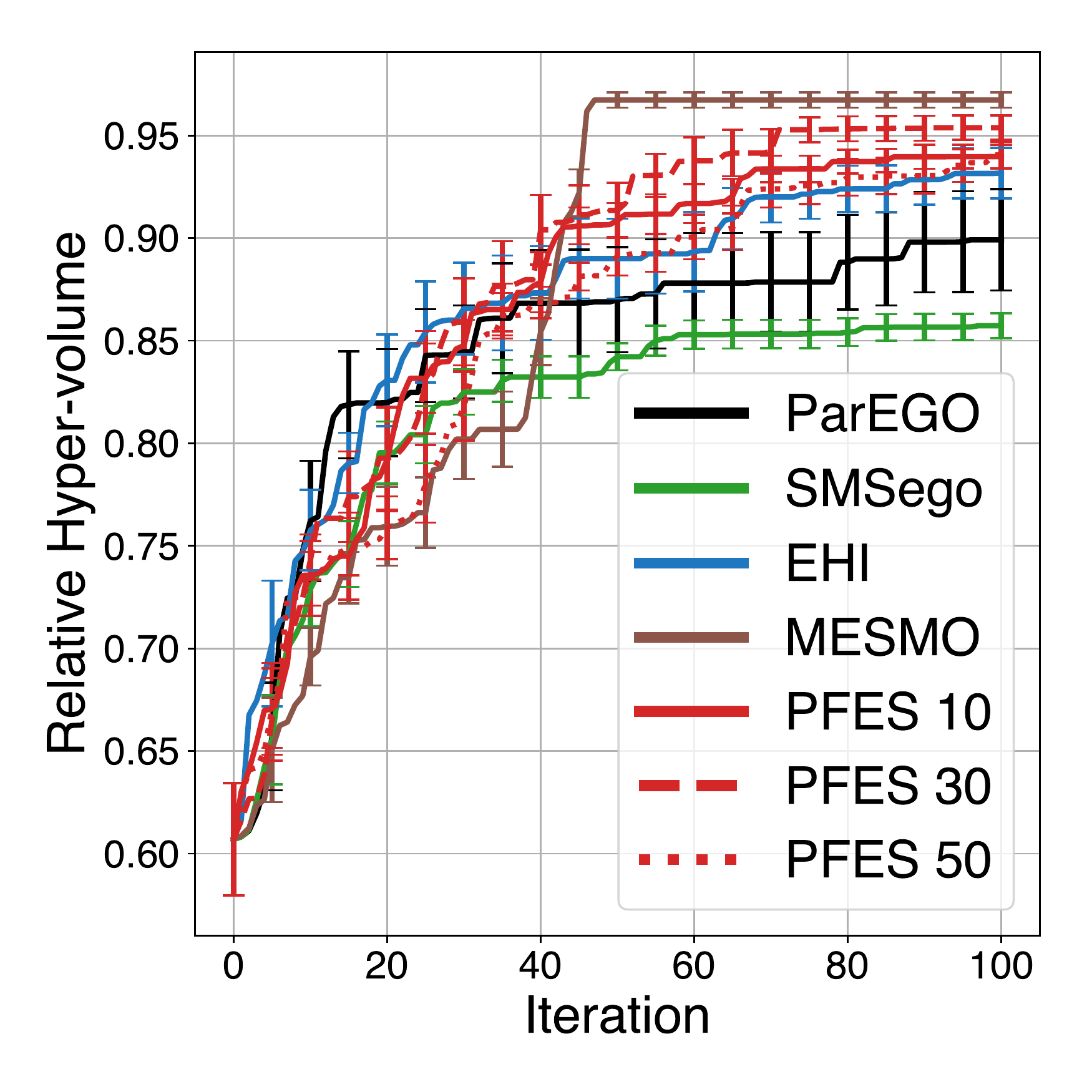}
 }

 \subfloat[DTLZ3]{
 \igr{0.24}{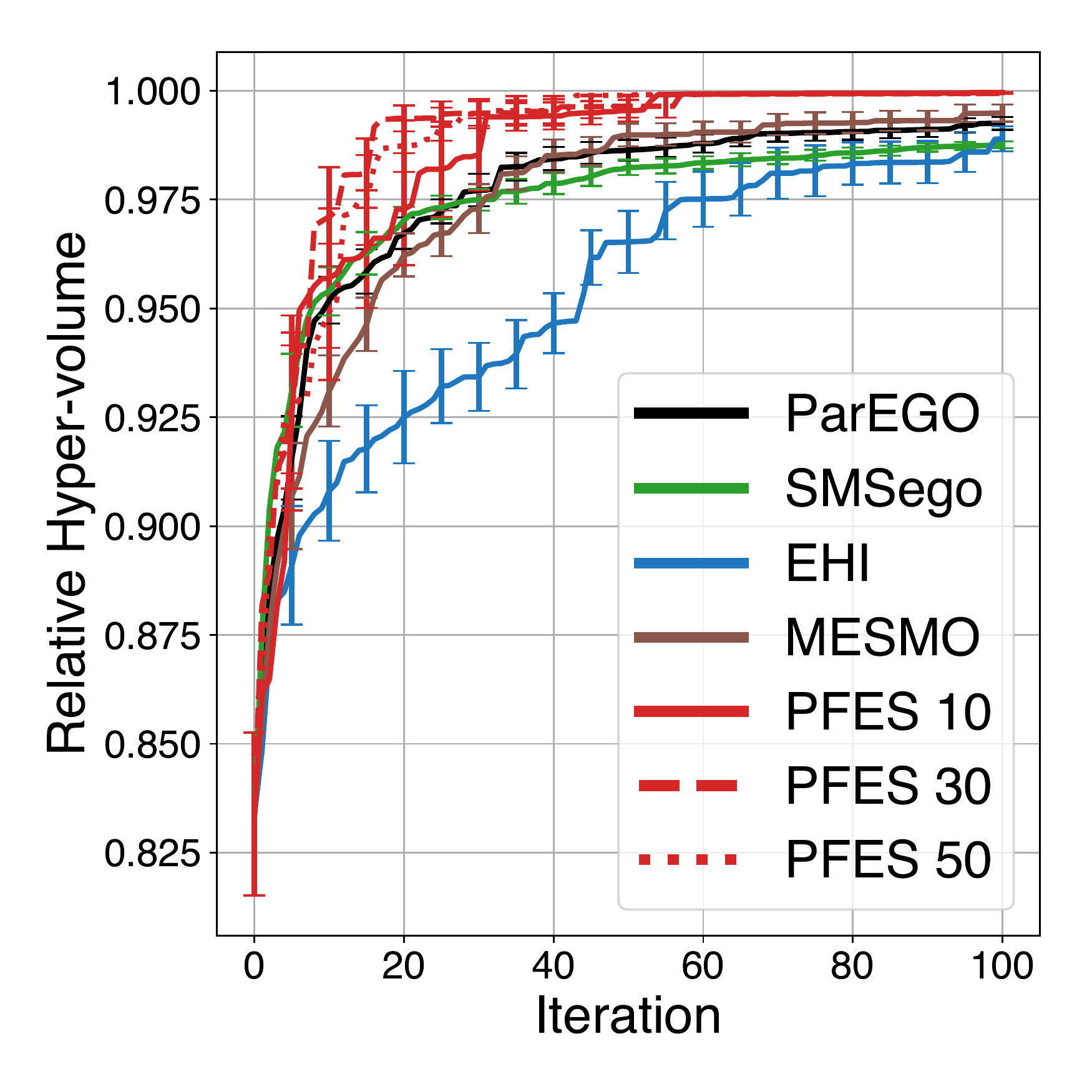}
 }
 \subfloat[DTLZ4]{
 \igr{0.24}{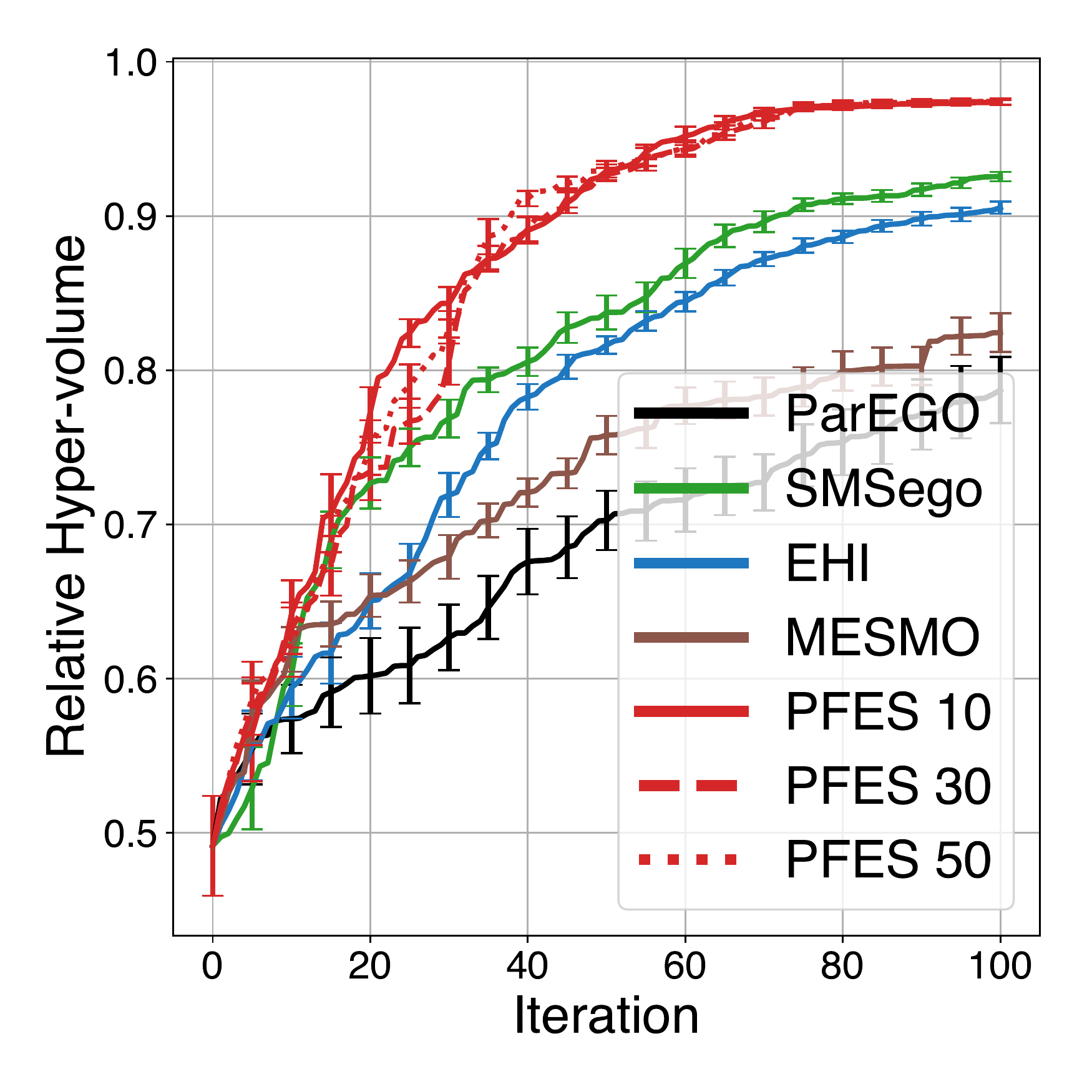}
 }
 \caption{
 Performance comparison on benchmark problems (average and standard error of $10$ runs).
 }
 \label{fig:benchmark} 
\end{figure}

% In \figurename~\ref{fig:benchmark}, we see that PFES rapidly increases SRHV compared with the other approaches.
% %
% ParEGO shows reasonable performance on Ackley/Sphere and DTLZ3, but in ZDT4 and DTLZ4, the increase was quite slow.
% %
% This performance variance may be caused by effect of scalarization.
% %
% Although EHI was relatively stable, our approach outperformed it except for the beginning of ZDT4.
% %
% SMSego also shown comparable performance with EHI except for ZDT4.
% %
% PFES also outperformed PESMO for all datasets which suggests that the entropy of Pareto-frontier $\cF^*$ can be a practical alternative of the entropy of Pareto-set $\cX^*$.

We also examined the computational time for the acquisition function evaluation on DTLZ4 which has the largest output dimension $L = 4$ in our four benchmark functions.
We randomly selected $50$ training instances and also randomly selected $100$ candidate $\*x$ to evaluate the acquisition functions.
The average of $10$ runs of this procedure is shown in \tablename~\ref{tab:time} (a).
ParEGO and SMSego are fast because their acquisition functions are simple.
Although EHI took relatively long time, this is mainly because we employed the na{\"i}ve cell partitioning described in \citep{Shah2016-Pareto}.
MESMO and PFES are similar computational times. 
\tablename~\ref{tab:time} (b) shows the detailed elapsed time in PFES.
We see that the most of time was spent by NSGA-II in this case. 
The amount of QHV and the entropy calculation is quite small, and \#cells is less than 1,000 even in this $L = 4$ problem which is a middle-large sized MOO problem because a problem $L \geq 4$ is sometimes called a ``many-objective'' problem in the context of MOO \citep{Chand2015-Evolutionary}.
Since MESMO also employed NSGAII \citep{Belakaria2019-Max}, MESMO and PFES showed the similar result.
% We also confirmed the number of generated cells in PFES using DTLZ3 and DTLZ4 which have the largest output dimension $L = 4$ in four datasets.
%
% We randomly selected $50$ training instances and calculated the PFES acquisition function in which Pareto-frontier $\cF^*$ is generated $10$ times.
%
% Each $\cF^*$ contain $50$ Pareto optimal $\*f_{\*x}$ as we set in the above experiment.
%
% Then, the average number of cells and its standard deviation were $563.9 \pm 54.128$ and $739.4 \pm 122.673$, respectively.
%
% These would be tractably small though the worst case evaluation is exponential with respect to $L$.
%
% Since in most of problems, $L$ is quite small (typically, $2$ or $3$), PFES would be feasible for many practical settings.
%
% 
We further report with other settings in Appendix~\ref{app:time}.

% --------------------------------------------------
% Time Comparison on DTLZ4
% --------------------------------------------------
\setlength{\textfloatsep}{10pt}
\begin{table}[t]
 \centering
 \caption{
 Computational time for acquisition function evaluation for $100$ points on DTLZ4.
 }
 \label{tab:time}
 \subfloat[Comparison of five methods (sec)]{
 \begin{tabular}{|c|c|c|c|c|}
ParEGO  & SMSego & EHI & MESMO & PFES \\ \hline
0.53  & 6.32 & 317.55 & 59.90 & 62.36 \\
 \end{tabular}}

 % \caption{}
 \subfloat[Details of PFES (sec, except for \# cells)]{
 \begin{tabular}{|c||c|c|c|c||c|}
  Total  & RFM & NSGA-II & QHV & Entropy & \#cells \\ \hline
  62.36  & 0.11 & 59.85  & 1.28 & 1.13 & 647.86 \\
 \end{tabular}}
\end{table}

% --------------------------------------------------
\subsection{Decoupled Setting with Materials Data}

For evaluating the decoupled acquisition function, we used two real-world datasets from \emph{computational materials science}.
In this field, efficient exploration of materials is strongly demanded because accurate physical simulations are often computationally extremely expensive, in which simulations taking more than several days are common.
The task is to explore crystal structures achieving high ion-conductivity and stability (i.e., $L = 2$), which are desirable properties for battery materials.
For these datasets, $\cX$ is a pre-defined discrete set, meaning that we have the fixed number of candidates (the pooled setting).
Details of the two datasets, called Bi$_2$O$_3$ and LLTO, are as follows:
\vspace{-1em}
\begin{description}
 \setlength{\itemsep}{0pt}
 \setlength{\parskip}{0pt}
 \item[Bi$_2$O$_3$] The size of candidates is $|\cX| = 335$, generated by the composition Bi$_{1-x-y-z}$Er$_x$Nb$_y$W$_z$O$_{48+y+3/2z}$.
	    The input is the three dimensional space defined by $x$, $y$, and $z$.
 \item[LLTO] The size of candidates is $|\cX| = 1119$, generated by the crystal called Perovskite type La$_{2/3-x}$Li$_{3x}$TiO$_3$ for $x = 0.11$.
	    In each candidate, positions of each one of atoms are permuted. 
	    The $2185$ dimensional feature vector $\*x$ is created through relative three dimensional positions of the atoms.
	    % each one of atoms  has different 3D dimensional positions of the atoms in the crystal, by which the $2185$ dimensional feature vector $\*x$ is created.
	    %
	    Note that although this dataset has the high dimensional input space, BO is feasible because $\cX$ is the pre-defined discrete set.
\end{description}
\vspace{-1em}
The objective functions are ion-conductivity $f^1_{\*x}$ and stability $f^2_{\*x}$ (negative of the energy), which can be observed through physical simulation models, separately.
The Bi2O3 and LLTO data are collected based on quantum- and classical- mechanics, respectively.
% For the Bi$_2$O$_3$ data, a quantum mechanics calculation was used, and for the LLTO data, a classical mechanics calculation was used.}
%
In the both cases, ion-conductivity is more expensive because it requires time-consuming simulations for observing dynamics of the ion.
%
% Here, we set the observation cost of the ion-conductivity and stability as $\lambda_1 = 5$ and $\lambda_2 = 1$.
Here, we examine the two cost settings  
$(\lambda_1, \lambda_2) = (5, 1)$
and
$(\lambda_1, \lambda_2) = (10, 1)$, 
based on the prior knowledge of the domain experts.
%
% For evaluation, we used \emph{inference relative hyper-volume} (IRHV) because SRHV is not suitable to the decoupled methods in which only one of objectives are observed at every iteration.
%inspired by \emph{inference regret} which is often used   
% is the relative hyper-volume defined by 
%
% In IRHV, the Pareto set $\cX^*$ of the posterior mean function $\{ \*\mu(\*x) \}_{\*x \in \cX}$ is first identified, and then, the hyper-volume created by the true function on the identified Pareto set $\cX^*$ is evaluated.
%
In these datasets, PFES directly generated function values of GPR without RFM, from which the Pareto set can be easily sampled unlike the continuous input case.
%
% The decoupled variant of PFES used the cost-sensitive acquisition function \eq{eq:acq-single}, and PESMO can also define the same cost-sensitive variant of the original acquisition function (i.e., divided by the cost).
%
Each experiment run $10$ times with a different set of initial observations which were randomly selected $5$ points.

\figurename~\ref{fig:Bi2O3} and \ref{fig:LLTO} show the result.
The horizontal axis of the figure is the sum of the observation cost.
In Bi$_2$O$_3$, SMSego, EHI, PFES, and PFES (decoupled) showed relatively rapid convergence, while in LLTO,  PFES (decoupled) reached the maximum first.
Interestingly, for the both datasets, the increase of PFES (decoupled) was moderate compared with the other methods in the beginning, and it was accelerated at the middle of iterations. 
PFES (decoupled) starts sampling from low cost functions because in the beginning the amount of information from the two objectives are not largely different.
After collecting cheaper information, PFES (decoupled) moves onto the expensive objectives and the faster improvement of RHV compared with the coupled PFES was finally observed in a sense of the total sampling cost.
\setlength{\textfloatsep}{10pt}
\begin{figure}[t]
 \centering 
 \subfloat[$(\lambda_1, \lambda_2) = (5, 1)$]{
 \igr{0.24}{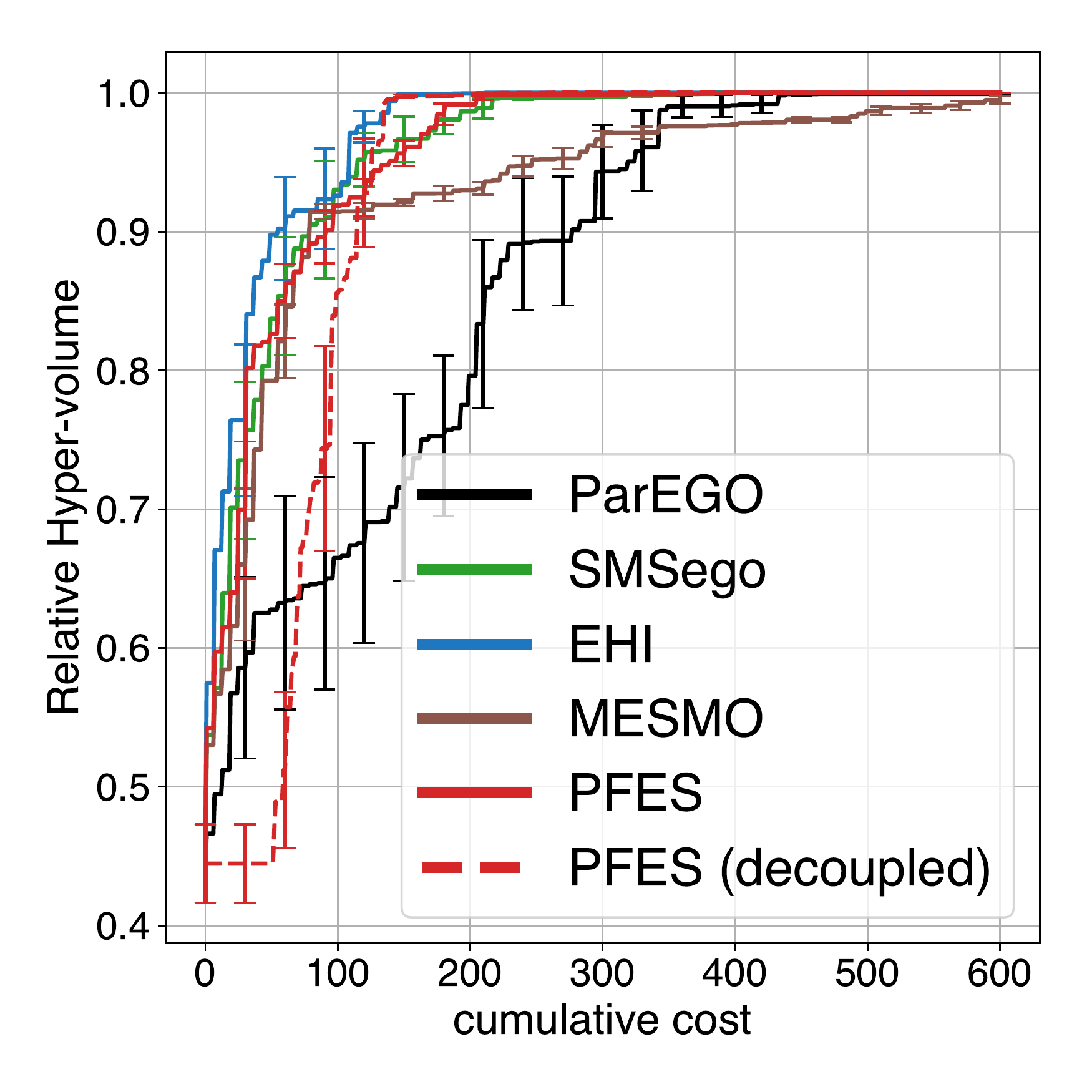}
 }
 \subfloat[$(\lambda_1, \lambda_2) = (10, 1)$]{
 \igr{0.24}{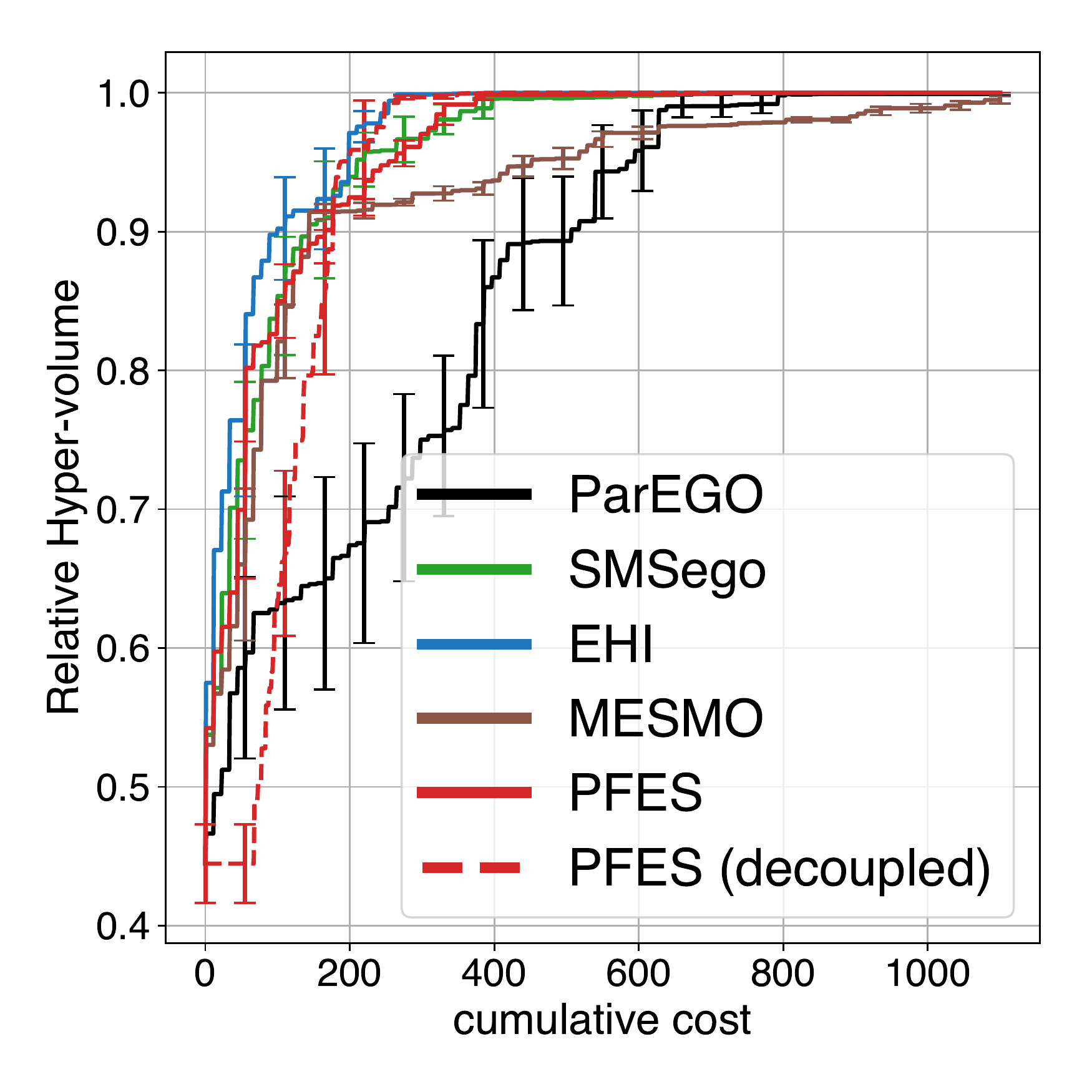}
 }
 \caption{
 RHV for Bi$_2$O$_3$.
 }
 \label{fig:Bi2O3}
 \subfloat[$(\lambda_1, \lambda_2) = (5, 1)$]{
 \igr{0.24}{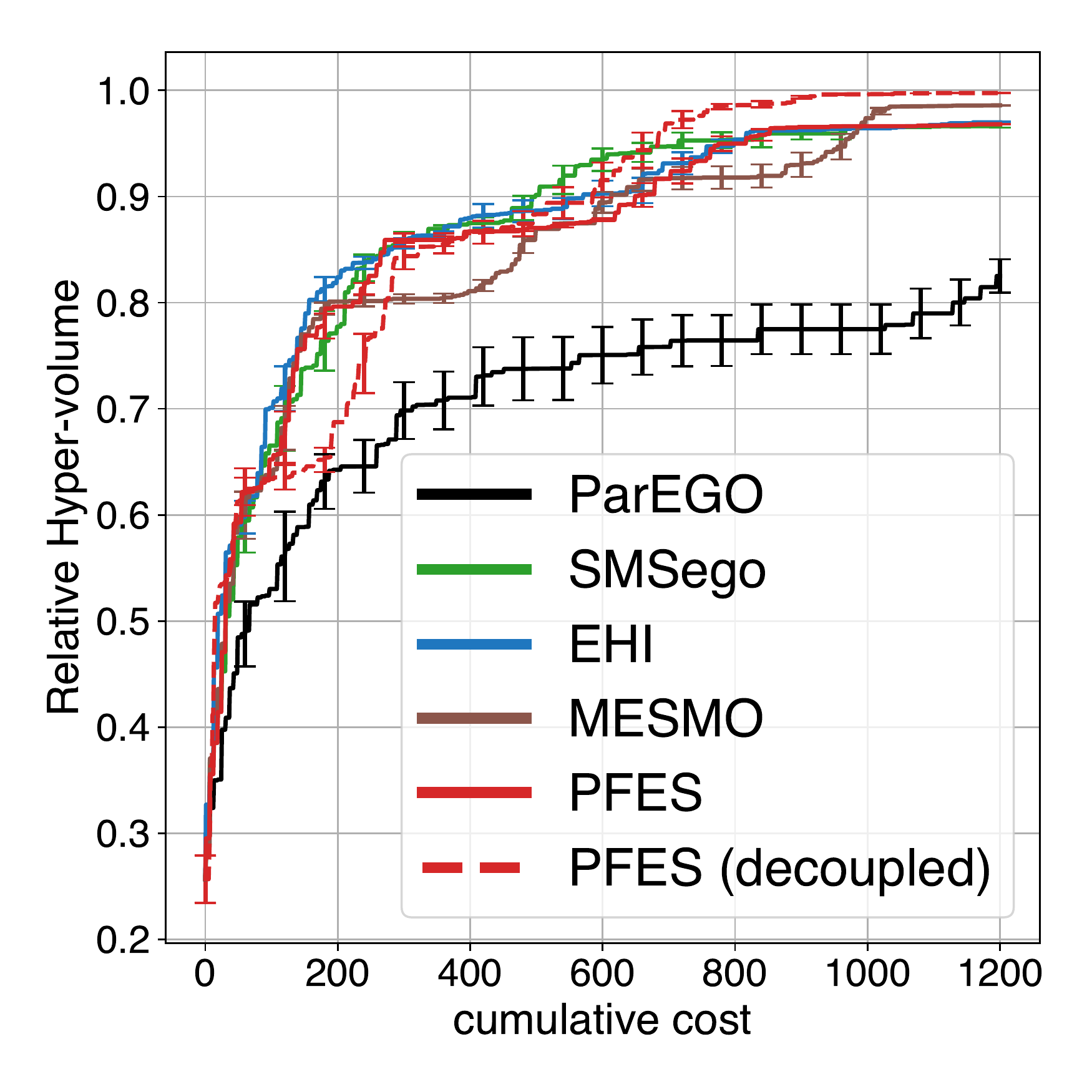}
 }
 \subfloat[$(\lambda_1, \lambda_2) = (10, 1)$]{
 \igr{0.24}{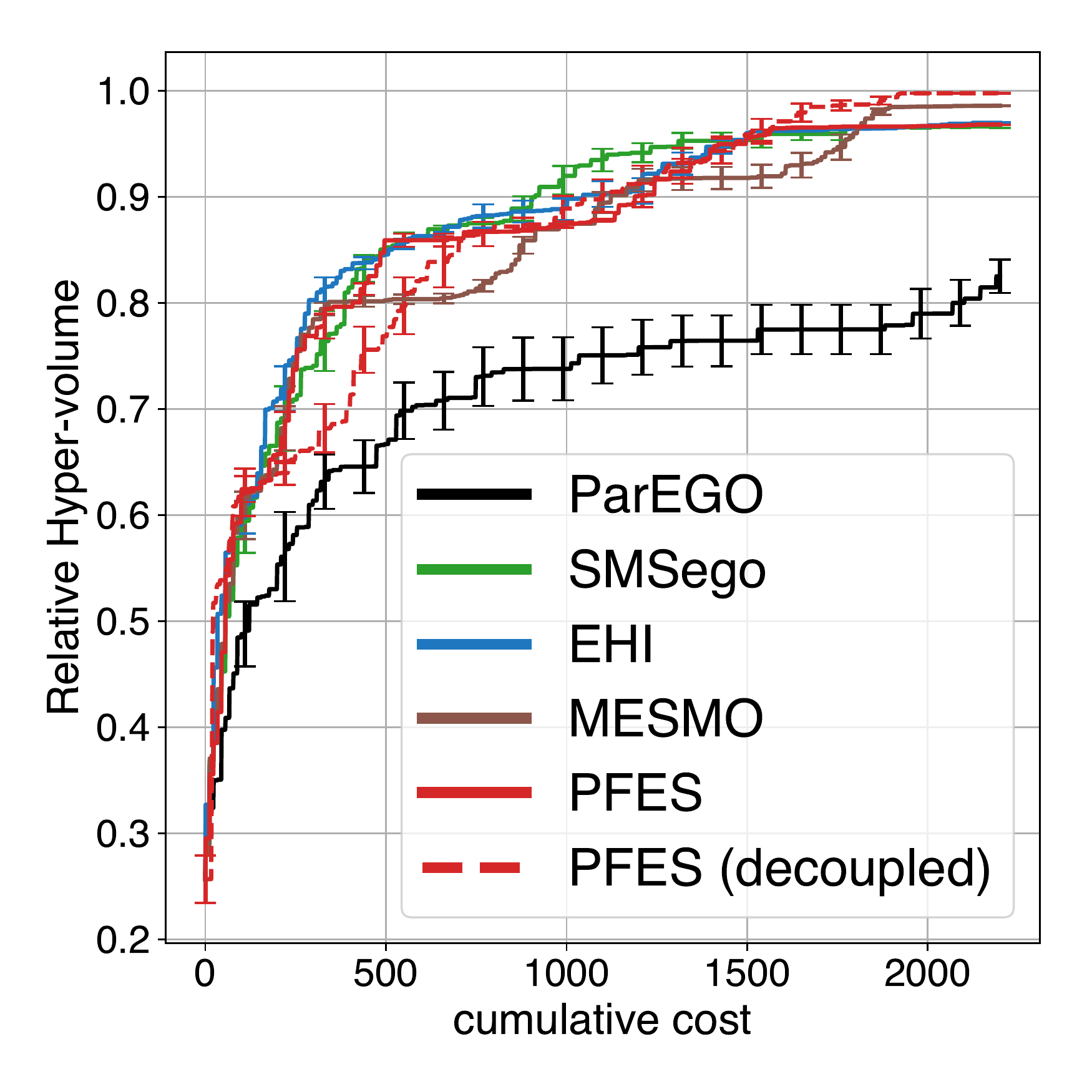}
 }
 \caption{
 RHV for LLTO.
 }
 \label{fig:LLTO}
\end{figure}

% --------------------------------------------------
\section{Conclusion}
\label{sec:conclusion}

We proposed \emph{Pareto-frontier entropy search} (PFES) for multi-objective Bayesian optimization (MBO).
We showed that the entropy of Pareto-frontier can be simply evaluated via sampling of Pareto-frontier and the cell-based partitioning.
Further, we showed PFES for the decoupled setting through the marginalization, for which simple computations are also obtained.
Our empirical evaluation on the benchmark functions and materials science data demonstrated effective of our approach.


% --------------------------------------------------
% \section*{References}
% --------------------------------------------------

% \small

% --------------------------------------------------
% References
% --------------------------------------------------

\bibliography{ref}
\bibliographystyle{icml2020}

\clearpage

\onecolumn 

% \begin{center}
%  {\Large {\bf 
%  Supplementary Materials for the Submission:\\
%  ``Multi-objective Bayesian Optimization using\\ Pareto-frontier Entropy''
%  }
%  }
% \end{center}
% \vspace{1em}

% --------------------------------------------------
\appendix
% --------------------------------------------------

% --------------------------------------------------
\section{Proof of Theorem~\ref{thm:entropy-reduction}}
\label{sec:proof-theorem1}

The normalization constant $Z$ is written as
% $Z$ is written as
\begin{align}
 Z \coloneqq 
 \int_{\cF} p(\*f_{\*x} \mid \cD) {\rm d} \*f_{\*x}
 =
 \sum_{m = 1}^M
 \int_{\cC_m} p(\*f_{\*x} \mid \cD)
 {\rm d} \*f_{\*x}
 = 
 \sum_{m = 1}^M
 \prod_{l = 1}^L 
 \int_{\ell^l_m}^{u^l_{m}}
 p(f^l_{\*x} \mid \cD)
 {\rm d} f^l_{\*x}	
 =
 \sum_{m = 1}^M
 Z_m.
 \label{eq:Z}  	
\end{align}
which is a sum of the Gaussian integrals in the cells.
%
% The entropy of PFTN is also decomposed into 
% \begin{align}
%  &
%  H[ p(\*f_{\*x} \mid \cD, \*f_{\*x} \preceq \cF^*) ] 
%  \nonumber \\
%  &=
%  - 
%  \int_{\cF}
%  \frac
%  {p(\*f_{\*x} \mid \cD) }
%  {Z}
%  \log 
%  \frac{p(\*f_{\*x} \mid \cD) }{Z}
%  {\rm d} \*f_{\*x}
%  \nonumber \\
%  & =
%  - 
%  \frac{1}	
%  {Z}
%  \sum_{m = 1}^M
%  \int_{\cC_m}
%  p(\*f_{\*x} \mid \cD)	
%  \log
%  p(\*f_{\*x} \mid \cD)	
%  {\rm d} \*f_{\*x}
%  +
%  \log {Z}.
%  \label{eq:partitioned-entropy}
% \end{align}

First, \eq{eq:Z-decomposed} is immediately derived from the independence of $L$ GPRs.
Let
$\cF \coloneqq \{  \*f_{\*x} \mid \*f_{\*x} \preceq \cF^* \}$.
Using
\begin{align*}
 % $p(\*f_{\*x} \mid \cD, \*f_{\*x} \in \cF)$,
 % p(\*f_{\*x} \mid \cD, \*f_{\*x} \preceq \cF^*) =
 p(\*f_{\*x} \mid \cD, \*f_{\*x} \in \cF) =
 \begin{cases}
  \frac{1}{Z}
  p(\*f_{\*x} \mid \cD), & 
  \text{ if } \*f_{\*x} \in \cF, \\
  0, & \text{ otherwise },
 \end{cases}
\end{align*}
we see
\begin{align}
 H[p(\*f_{\*x} \mid \cD, \*f_{\*x} \in \cF)]
 &=
 - \int_{\cF}
 \frac
 {p(\*f_{\*x} \mid \cD) }
 {Z}
 \log 
 \frac{p(\*f_{\*x} \mid \cD) }{Z}
 {\rm d} \*f_{\*x} 
 \nonumber
 \\
 &=
 - Z^{-1}
 \int_{\cF}
 p(\*f_{\*x} \mid \cD) 
 \log 
 p(\*f_{\*x} \mid \cD) 
 {\rm d} \*f_{\*x} 
 +
 Z^{-1} \log Z
 \int_{\cF}
 p(\*f_{\*x} \mid \cD) 
 {\rm d} \*f_{\*x}  
 \nonumber
 \\
 &=
 - Z^{-1}
 \int_{\cF}
 p(\*f_{\*x} \mid \cD) 
 \log 
 p(\*f_{\*x} \mid \cD) 
 {\rm d} \*f_{\*x} 
 +
 \log Z
 \nonumber
 \\
 &=
 - Z^{-1}
 \sum_{m = 1}^M
 \int_{\cC_m}
 p(\*f_{\*x} \mid \cD) 
 \log 
 p(\*f_{\*x} \mid \cD) 
 {\rm d} \*f_{\*x} 
 +
 \log Z
 \label{eq:conditional-entropy-ind}
\end{align}
Based on the independence of $\*f_{\*x}$, the integral of the first term can be transformed into
\begin{align}
 & 
 % \int_{\cF}
 % p(\*f_{\*x} \mid \cD) 
 % \log 
 % p(\*f_{\*x} \mid \cD) 
 % {\rm d} \*f_{\*x} 
 % \nonumber
 % \\ % -------------------------
 % & =
 \sum_{m = 1}^M
 \int_{\cC_m}
 p(\*f_{\*x} \mid \cD) 
 \log 
 p(\*f_{\*x} \mid \cD) 
 {\rm d} \*f_{\*x} 
 \nonumber
 \\ % -------------------------
 & =
 \sum_{m = 1}^M
 \int_{\ell^1_m}^{u^1_m}
 \int_{\ell^2_m}^{u^2_m}
 \cdots
 \int_{\ell^L_m}^{u^L_m}
 \prod_{l' = 1}^L p(f^{l'}_{\*x} \mid \cD)
 \left(
 \sum_{l = 1}^L
 \log 
 p(f^l_{\*x} \mid \cD)
 \right)
 {\rm d} f^L_{\*x}	
 \cdots
 {\rm d} f^1_{\*x}
 \nonumber
 \\ % -------------------------
 & =
 \sum_{m = 1}^M
 \sum_{l = 1}^L
 \int_{\ell^1_m}^{u^1_m}
 \int_{\ell^2_m}^{u^2_m}
 \cdots
 \int_{\ell^L_m}^{u^L_m}
 \prod_{l' = 1}^L p(f^{l'}_{\*x} \mid \cD)
 \log 
 p(f^l_{\*x} \mid \cD)
 {\rm d} f^L_{\*x}	
 \cdots
 {\rm d} f^1_{\*x}
 \nonumber
 \\ % -------------------------
 & =
 \sum_{m = 1}^M
 \sum_{l = 1}^L
 \left[
 \left(
 \int_{\ell^l_m}^{u^l_m}
 p(f^l_{\*x} \mid \cD)
 \log 
 p(f^l_{\*x} \mid \cD)
 {\rm d} f^l_{\*x}	
 \right)
 \left(
 \prod_{l' \neq l} 
 \int_{\ell^{l'}_m}^{u^{l'}_m}
 p(f_{l'}(\*x) \mid \cD)
 {\rm d} f_{l'}(\*x)	
 \right)
 \right]
 \nonumber
 \\ % -------------------------
 & =
 \sum_{m = 1}^M
 \sum_{l = 1}^L
 \left[
 Z_{ml}
 \int_{\ell^l_m}^{u^l_m}
 \left(
 \frac{p(f^l_{\*x} \mid \cD)}	      
 {Z_{ml}}
 \log 
 \frac{p(f^l_{\*x} \mid \cD)}	     
 {Z_{ml}}
 +
 \frac{p(f^l_{\*x} \mid \cD)}	      
 {Z_{ml}}
 \log Z_{ml}
 \right)
 {\rm d} f^l_{\*x}	
 \prod_{l' \neq l} 
 Z_{m l'}
 \right]
 \nonumber
 \\ % -------------------------
 & =
 \sum_{m = 1}^M
 \sum_{l = 1}^L
 \left[	       
 \Biggl(
 Z_{ml}
 \underbrace{
 \int_{\ell^l_m}^{u^l_m}
 \frac{p(f^l_{\*x} \mid \cD)}	      
 {Z_{ml}}
 \log 
 \frac{p(f^l_{\*x} \mid \cD)}	      
 {Z_{ml}}
 {\rm d} f^l_{\*x}	
 }_{(\star)} 
 +
 Z_{ml}
 \log Z_{ml}
 \Biggr)
 \prod_{l' \neq l} 
 Z_{m l'}
 \right]	       
 \label{eq:unnormalized-entropy}
\end{align}
The term indicated by $\star$ is the negative entropy of the truncated normal distribution.
For the entropy of the truncated normal distribution, analytical formula is available \citep[e.g,][]{Michalowicz2013-Handbook}, by which we can obtain
\begin{align*}
 \int_{\ell^l_m}^{u^l_m}
 \frac{p(f^l_{\*x} \mid \cD)}	      
 {Z_{ml}}
 \log 
 \frac{p(f^l_{\*x} \mid \cD)}	      
 {Z_{ml}}
 {\rm d} f^l_{\*x}	 
 =
 -
 \log (\sqrt{2 \pi e} \sigma_l(\*x) Z_{ml})
 -
 \frac
 {\ubar{\alpha}_{m,l} \phi(\ubar{\alpha}_{m,l}) - \bar{\alpha}_{m,l} \phi(\bar{\alpha}_{m,l})}
 {2 Z_{ml}}.
\end{align*}
% where $\ubar{\alpha}_{m,l} \coloneqq (f_l^m - \mu_l(\*x)) / \sigma_l(\*x)$.
%
Then, the above equation \eqref{eq:unnormalized-entropy} is further transformed into	     
\begin{align*}
 % &
 % \sum_{m = 1}^M
 % \sum_{l = 1}^L
 % \left[	       
 % \Biggl(
 % Z_{ml}
 % \int_{f_l^m}^{f_l^{m+1}}
 % \frac{p(f^l_{\*x} \mid \cD)}	      
 % {Z_{ml}}
 % \log 
 % \frac{p(f^l_{\*x} \mid \cD)}	      
 % {Z_{ml}}
 % {\rm d} f^l_{\*x}	
 % +
 % Z_{ml}
 % \log Z_{ml}
 % \Biggr)
 % \prod_{l' \neq l} 
 % Z_{m l'}
 % \right]	       	       
 % \\
 % & =
 &
 \sum_{m = 1}^M
 \sum_{l = 1}^L
 \left[	       
 \Biggl(
 Z_{ml}
 \Bigl(-
 \log (\sqrt{2 \pi e} \sigma_l(\*x) Z_{ml})
 -
 \frac
 {\ubar{\alpha}_{m,l} \phi(\ubar{\alpha}_{m,l}) - \bar{\alpha}_{m,l} \phi(\bar{\alpha}_{m,l})}
 {2 Z_{ml}}
 \Bigr)
 +
 Z_{ml}
 \log Z_{ml}
 \Biggr)
 \prod_{l' \neq l} 
 Z_{m l'}
 \right]	       	       
 \\
 &=
 \sum_{m = 1}^M
 \sum_{l = 1}^L
 \left[	       
 \Biggl(
 -
 Z_{ml}
 \log (\sqrt{2 \pi e} \sigma_l(\*x))
 -
 \frac
 {\ubar{\alpha}_{m,l} \phi(\ubar{\alpha}_{m,l}) - \bar{\alpha}_{m,l} \phi(\bar{\alpha}_{m,l})}
 {2}
 \Biggr)
 \prod_{l' \neq l} 
 Z_{m l'}
 \right].	       	    
\end{align*}
Substituting this into \eqref{eq:conditional-entropy-ind}, we obtain
\begin{align*}
 & H[p(\*f_{\*x} \mid \cD, \*f_{\*x} \in \cF)]
 % \\
 % & =
 % - Z^{-1}
 % \int_{\cF}
 % p(\*f_{\*x} \mid \cD) 
 % \log 
 % p(\*f_{\*x} \mid \cD) 
 % {\rm d} \*f_{\*x} 
 % +
 % \log Z
 \\
 & =
 - Z^{-1}
 \sum_{m = 1}^M \sum_{l = 1}^L
 \left[	\Biggl(
 -
 Z_{ml}
 \log (\sqrt{2 \pi e} \sigma_l(\*x))
 -
 \frac
 {\ubar{\alpha}_{m,l} \phi(\ubar{\alpha}_{m,l}) - \bar{\alpha}_{m,l} \phi(\bar{\alpha}_{m,l})}
 {2}
 \Biggr)
 \prod_{l' \neq l} 
 Z_{m l'}
 \right]	       	       
 +
 \log Z 
 \\
 & =
 Z^{-1} \sum_{m = 1}^M \prod_{l'=1}^L Z_{ml'}
 \sum_{l = 1}^L
 \log
 \left(
 \sqrt{2 \pi e} \sigma_l (\*x)
 \right)
 + 
 Z^{-1}
 \sum_{m = 1}^M
 \sum_{l = 1}^L
 \frac
 {\ubar{\alpha}_{m,l} \phi(\ubar{\alpha}_{m,l}) - \bar{\alpha}_{m,l} \phi(\bar{\alpha}_{m,l})}
 {2}
 \prod_{l' \neq l} 
 Z_{m l'}
 +
 \log Z
 \\
 & =
 \log
 \left(
 (\sqrt{2 \pi e})^L
 Z
 \prod_{l=1}^L \sigma_l(\*x) 
 \right)
 +
 Z^{-1}
 \sum_{m = 1}^M
 \sum_{l = 1}^L
 \prod_{l' \neq l} Z_{m l'}
 \frac
 {\ubar{\alpha}_{m,l} \phi(\ubar{\alpha}_{m,l}) - \bar{\alpha}_{m,l} \phi(\bar{\alpha}_{m,l})}
 {2}	      
 \\
 & =
 \log
 \left(
 (\sqrt{2 \pi e})^L
 Z
 \prod_{l=1}^L \sigma_l(\*x) 
 \right)
 +
 % Z^{-1}
 \sum_{m = 1}^M
 % Z_{m}
 \frac{Z_{m}}{Z}
 \sum_{l = 1}^L
 % \prod_{l' \neq l} Z_{m l'}
 \frac
 {\ubar{\alpha}_{m,l} \phi(\ubar{\alpha}_{m,l}) - \bar{\alpha}_{m,l} \phi(\bar{\alpha}_{m,l})}
 {2 Z_{ml}}.	      
\end{align*}
% where $Z_m = \prod_{l=1}^L Z_{ml}$. 

% --------------------------------------------------
\section{Proof of Theorem~\ref{thm:single-entropy-ind}}
\label{sec:proof-theorem2}

The marginalization can be represented as
\begin{align}
 % &
 % p(f^l_{\*x} \red{\in (\tilde{f}_l^m, \tilde{f}_l^{m+1}]} \mid \cD, \*f_{\*x} \preceq \cF^*) 
 & 
 p(f^l_{\*x} \mid \cD, \*f_{\*x} \preceq \cF^*) 
 \nonumber
 \\
 & =
 \sum_{m \in \cM(l,s_l^{(f^l_{\*x})})}
 \int_{\cC_{m}^{{\setminus l}}}
 \frac{p(\*f_{\*x} \mid \cD)}{Z}
 {\rm d} \*f^{{\setminus l}}_{\*x} 
 \nonumber
 \\
 & =
 \frac{p(f_{\*x}^l \mid \cD)}{Z}
 \sum_{m \in \cM(l,s_l^{(f^l_{\*x})})}
 \int_{\cC_{m}^{{\setminus l}}}
 p(\*f^{{\setminus l}}_{\*x} \mid f_{\*x}^l, \cD) 
 {\rm d} \*f^{{\setminus l}}_{\*x},		
 \label{eq:marginal-given-frontier}
\end{align}
where 
$\cC_{m}^{{\setminus l}}$ is the $(L-1)$-dimensional cell created by eliminating the $l$-th dimension of $\cC_{m}$, and $\*f^{{\setminus l}}_{\*x}$ is a subvector of $\*f_{\*x}$ without the $l$-th dimension.

The marginal distribution of $f_{\*x}^l$ can be partitioned into an interval $f^l_{\*x} \in (\tilde{f}_l^m, \tilde{f}_l^{m+1}]$ as shown in \eqref{eq:marginal-given-frontier}, which can be further transformed into
\begin{align*}
 &
 % p(f^l_{\*x} \in (\tilde{f}_l^m, \tilde{f}_l^{m+1}] \mid \cD, \*f_{\*x} \preceq \cF^*) 
 p(f^l_{\*x} \in (\tilde{f}_l^s, \tilde{f}_l^{s+1}] \mid \cD, \*f_{\*x} \preceq \cF^*) 
 \\ % -------------------------
 &=
 % \sum_{m' \in \cM_{l,m}}
 \sum_{m' \in \cM(l,s_l^{(f_{\*x}^l)}) }
 \int_{\cC_{m'}^{{\setminus l}}}
 p(\*f_{\*x} \mid \cD, \*f_{\*x} \preceq \cF^*) 
 {\rm d} \*f_{{\setminus l}}(\*x) 
 \\ % -------------------------
 &=
 \frac{1}{Z}
 % \sum_{m' \in \cM_{l,m}}
 \sum_{m' \in \cM(l,s_l^{(f_{\*x}^l)}) }
 \int_{\cC_{m'}^{{\setminus l}}}
 p(\*f_{\*x} \mid \cD) 
 {\rm d} \*f_{{\setminus l}}(\*x)	
 \\ % -------------------------
 &=
 \frac{1}{Z}
 % \sum_{m' \in \cM_{l,m}}
 \sum_{m' \in \cM(l,s_l^{(f_{\*x}^l)}) }
 \int_{\cC_{m'}^{{\setminus l}}}
 \prod_{l' \neq l} p(f_{l'}(\*x) \mid \cD) 
 p(f^l_{\*x} \mid \cD) 
 {\rm d} \*f_{{\setminus l}}(\*x)	
 % \text{ \# $BFHN)@-$h$j(B}
 \\ % -------------------------
 &=
 \frac{1}{Z}
 p(f^l_{\*x} \mid \cD) 
 % \sum_{m' \in \cM_{l,m}}
 \sum_{m' \in \cM(l,s_l^{(f_{\*x}^l)}) }
 \prod_{l' \neq l}
 (
 % \Phi(\alpha_{m'+1,l'}) - \Phi(\alpha_{m',l'})
 \Phi(\bar{\alpha}_{m',l'}) - \Phi(\ubar{\alpha}_{m',l'})
 )
 % \int_{\cC_m^{{\setminus l}}}
 % \prod_{l' \neq l} p(f_{l'}(\*x) \mid \cD) 
 % {\rm d} \*f_{{\setminus l}}(\*x)	
 \\ % -------------------------
 &=
 \frac{ 
 % \sum_{m' \in \cM_{l,m}}
 \sum_{m' \in \cM(l,s_l^{(f_{\*x}^l)}) }
 \prod_{l' \neq l} Z_{m'l'}
 }{Z}
 p(f^l_{\*x} \mid \cD) 
\end{align*}      
Let
$\tilde{f}_l^{0} = -\infty$, 
for convenience.
Then, the entropy is 
\begin{align}
 & 
 H[ p(f^l_{\*x} \mid \cD, \*f_{\*x} \preceq \cF^*) ]
 \nonumber
 \\ % -------------------------
 & =
 - \int_{-\infty}^{\infty} %{\tilde{f}_l^{M_l}}
 p(f^l_{\*x} \mid \cD, \*f_{\*x} \preceq \cF^*)
 \log
 p(f^l_{\*x} \mid \cD, \*f_{\*x} \preceq \cF^*)
 {\rm d} f^l_{\*x} 
 \nonumber
 \\ % -------------------------
 & =
 \sum_{s = 0}^{S_l-1}
 \int_{\tilde{f}_l^{s}}^{\tilde{f}_l^{s+1}}
 p(f^l_{\*x}  \mid \cD, \*f_{\*x} \preceq \cF^*)
 \log
 p(f^l_{\*x}  \mid \cD, \*f_{\*x} \preceq \cF^*)
 {\rm d} f^l_{\*x} 
 \nonumber
 \\ % -------------------------
 & =
 - \sum_{s = 0}^{S_l-1}
 \int_{\tilde{f}_l^{s}}^{\tilde{f}_l^{s+1}}
 \frac{ 
 \sum_{m' \in \cM(l,s)}
 \prod_{l' \neq l} Z_{m'l'}
 }{Z}
 p(f^l_{\*x} \mid \cD) 
 \log
 \frac{ 
 \sum_{m' \in \cM(l,s)}
 \prod_{l' \neq l} Z_{m'l'}
 }{Z}
 p(f^l_{\*x} \mid \cD) 
 {\rm d} f^l_{\*x} 
 \nonumber
 \\ % -------------------------
 % & =
 % - \sum_{s = 0}^{S_l-1}
 % \frac{ 
 % \sum_{m' \in \cM(l,s)}
 % \prod_{l' \neq l} Z_{m'l'}
 % }{Z}
 % \int_{\tilde{f}_l^{s}}^{\tilde{f}_l^{s+1}}
 % p(f^l_{\*x} \mid \cD) 
 % \log
 % \left(
 % \frac{ 
 % \sum_{m' \in \cM(l,s)}
 % \prod_{l' \neq l} Z_{m'l'}
 % }{Z}
 % p(f^l_{\*x} \mid \cD) 
 % \right)
 % {\rm d} f^l_{\*x} 
 % \\ % --------------------------------------------------
 & =
 - \sum_{s = 0}^{S_l-1}
 \frac{ 
 \sum_{m' \in \cM(l,s)}
 \prod_{l' \neq l} Z_{m'l'}
 }{Z}
 \int_{\tilde{f}_l^{s}}^{\tilde{f}_l^{s+1}}
 p(f^l_{\*x} \mid \cD) 
 \left(
 \log
 \frac{ 
 \sum_{m' \in \cM(l,s)}
 \prod_{l' \neq l} Z_{m'l'}
 }{Z}
 +
 \log
 p(f^l_{\*x} \mid \cD) 
 \right)
 {\rm d} f^l_{\*x} 	
 \nonumber
 \\ % -------------------------
 &=
 - \sum_{s = 0}^{S_l-1}
 \frac{ 
 \sum_{m' \in \cM(l,s)}
 \prod_{l' \neq l} Z_{m'l'}
 }{Z}
 \Biggl(
 (\Phi(\tilde{\alpha}_{s+1,l}) - \Phi(\tilde{\alpha}_{s,l}))
 \log
 \frac{ 
 \sum_{m' \in \cM(l,s)}
 \prod_{l' \neq l} Z_{m'l'}
 }{Z}
 \nonumber
 \\
 & 
 \qquad +
 \int_{\tilde{f}_l^{s}}^{\tilde{f}_l^{s+1}}
 p(f^l_{\*x} \mid \cD) 
 \log
 p(f^l_{\*x} \mid \cD) 
 {\rm d} f^l_{\*x} 	
 \Biggr)
 \label{eq:single-entropy-ind-step1}
\end{align}
By transforming the last term in the parenthesis into the entropy of the truncated normal distribution, we see
\begin{align}
 &
 \int_{\tilde{f}_l^{s}}^{\tilde{f}_l^{s+1}}
 p(f^l_{\*x} \mid \cD) 
 \log p(f^l_{\*x} \mid \cD)
 {\rm d} f^l_{\*x} 
 \nonumber \\ % -------------------------
 &=
 \tilde{Z}_{sl}
 \int_{\tilde{f}_l^{s}}^{\tilde{f}_l^{s+1}}
 \frac{p(f^l_{\*x} \mid \cD) }{\tilde{Z}_{sl}}
 \left(
 \log \frac{p(f^l_{\*x} \mid \cD)}{\tilde{Z}_{sl}}
 +
 \log \tilde{Z}_{sl}	
 \right)
 {\rm d} f^l_{\*x} 
 \nonumber \\ % -------------------------
 &=	
 - \tilde{Z}_{sl}
 \left\{
 \log
 (\sqrt{2 \pi e} \sigma_{l}(\*x) \tilde{Z}_{sl})
 +
 \frac{ 
 {\tilde{\alpha}_{s,l} \phi(\tilde{\alpha}_{s,l}) 
 - \tilde{\alpha}_{s+1,l} \phi(\tilde{\alpha}_{s+1,l}) 
 }
 }{ 2 \tilde{Z}_{sl}}
 \right\}
 \nonumber \\
 & \qquad
 +
 \tilde{Z}_{sl}
 \int_{\tilde{f}_l^{s}}^{\tilde{f}_l^{s+1}}
 \frac{p(f^l_{\*x} \mid \cD) }{\tilde{Z}_{sl}}
 \log \tilde{Z}_{sl}	
 {\rm d} f^l_{\*x} 	
 \nonumber
 \\ % -------------------------
 & =
 - \tilde{Z}_{sl}
 \left\{
 \log
 (\sqrt{2 \pi e} \sigma_{l}(\*x) \tilde{Z}_{sl})
 +
 \frac{ 
 \tilde{\alpha}_{s,l} \phi(\tilde{\alpha}_{s,l}) 
 - \tilde{\alpha}_{s+1,l} \phi(\tilde{\alpha}_{s+1,l}) }{ 2 \tilde{Z}_{sl}}
 \right\}
 +
 \tilde{Z}_{sl}	
 \log \tilde{Z}_{sl}	
 \label{eq:unnormalized-truncated-ent}
\end{align}
By substituting this into \eqref{eq:single-entropy-ind-step1}, we obtain
\begin{align*}
 & 
 H[ p(f^l_{\*x} \mid \cD, \*f_{\*x} \preceq \cF^*) ] 
 \\ % -------------------------
 & =
 - \sum_{s = 0}^{S_l-1}
 \frac{ 
 \sum_{m' \in \cM(l,s)}
 \prod_{l' \neq l} Z_{m'l'}
 }{Z}
 \Biggl(
 \tilde{Z}_{sl}
 \log
 \frac{ 
 \sum_{m' \in \cM(l,s)}
 \prod_{l' \neq l} Z_{m'l'}
 }{Z}
 \\
 & 
 \qquad 
 % \int_{\tilde{f}_l^{s}}^{\tilde{f}_l^{s+1}} p(f^l_{\*x} \mid \cD) \log p(f^l_{\*x} \mid \cD) {\rm d} f^l_{\*x} 	
 - \tilde{Z}_{sl}
 \left\{
 \log
 (\sqrt{2 \pi e} \sigma_{l}(\*x) \tilde{Z}_{sl})
 +
 \frac{ 
 \tilde{\alpha}_{s,l} \phi(\tilde{\alpha}_{s,l}) 
 - \tilde{\alpha}_{s+1,l} \phi(\tilde{\alpha}_{s+1,l}) }{ 2 \tilde{Z}_{sl}}
 \right\}
 +
 \tilde{Z}_{sl}	
 \log \tilde{Z}_{sl}	
 \Biggr)
 \\ % -------------------------
 & =
 - \sum_{s = 0}^{S_l-1}
 \frac{ 
 \sum_{m' \in \cM(l,s)}
 \prod_{l' \neq l} Z_{m'l'}
 }{Z}
 \tilde{Z}_{sl}
 \Biggl(
 \log
 \frac{ 
 \sum_{m' \in \cM(l,s)}
 \prod_{l' \neq l} Z_{m'l'}
 }{Z}
 \\
 & 
 \qquad 
 - 
 \log
 (\sqrt{2 \pi e} \sigma_{l}(\*x) \tilde{Z}_{sl})
 -
 \frac{ 
 \tilde{\alpha}_{s,l} \phi(\tilde{\alpha}_{s,l}) 
 - \tilde{\alpha}_{s+1,l} \phi(\tilde{\alpha}_{s+1,l}) }{ 2 \tilde{Z}_{sl}}
 +
 % \tilde{Z}_{sl}	
 \log \tilde{Z}_{sl}	
 \Biggr)
\end{align*}
From the definition, 
if
$m' \in \cM(l,s)$,
then 
$\tilde{Z}_{sl} = Z_{m'l}$,
from which we obtain
$\sum_{m' \in \cM(l,s)} \left( \prod_{l' \neq l} Z_{m'l'} \right) \tilde{Z}_{sl} = \sum_{m' \in \cM(l,s)} \left( \prod_{l' = 1}^L Z_{m'l'} \right)  = \sum_{m' \in \cM(l,s)} Z_{m'}$.
This derives
\begin{align*}
 & 
 H[ p(f^l_{\*x} \mid \cD, \*f_{\*x} \preceq \cF^*) ] 
 \\ % -------------------------
 & =
 - \sum_{s = 0}^{S_l-1}
 \frac{ 
 \sum_{m' \in \cM(l,s)}
 Z_{m'}
 % \prod_{l' \neq l} Z_{m'l'}
 }{Z}
 % \tilde{Z}_{sl}
 \Biggl(
 \log
 \frac{ 
 \sum_{m' \in \cM(l,s)}
 Z_{m'}
 % \prod_{l' \neq l} Z_{m'l'}
 }{Z}
 \\
 & 
 \qquad 
 - 
 \log
 (\sqrt{2 \pi e} \sigma_{l}(\*x) {\tilde{Z}_{sl}})
 -
 \frac{ 
 \tilde{\alpha}_{s,l} \phi(\tilde{\alpha}_{s,l}) 
 - \tilde{\alpha}_{s+1,l} \phi(\tilde{\alpha}_{s+1,l}) }{ 2 \tilde{Z}_{sl}}
 % +	\log \tilde{Z}_{ml}	
 \Biggr)	
\end{align*}

% -------------------------
% \subsection{Extension to Decoupled Setting}
% \label{sec:decoupled}
\section{Extension to Correlated Objectives}
\label{app:correlated}

% \red{MEMO: Write semi-latent model}

Objective functions in MOO are often correlated each other.
Then, by incorporating the correlation into GPR, the search can be accelerated.
Several studies have considered constructing multiple correlated GPR models including \emph{multi-task GPR} model \citep{Bonilla2008-Multi} and \emph{semiparametric latent factor} (SLF) model \citep{Seeger04-Semiparametric}.
In the standard approaches including multi-task GPR and SLF, the multi-dimensional predictive distribution for $\*x$ is reduced to a multi-variate Gaussian distribution $\cN(\*\mu(\*x), \*\Sigma(\*x))$, where $\*\mu(\*x) \in \RR^L$ and $\*\Sigma(\*x) \in \RR^{L \times L}$ are the predictive mean and covariance matrix.
For considering an extension of PFES to correlated objectives, we assume that the surrogate model is represented as a GPR model jointly for multiple responses.
% {Objective functions in MOO are often correlated, and thus incorporating the correlation into GPR can accelerate the search.
% %
% \red{A typical approach is to use multi-task GPR \citep{Bonilla2008-Multi}.}
% %
% Let $k^f(l,l')$ be a kernel between objectives $l$ and $l'$, and $\*K^f$ be a matrix in which the $l,l'$-th element is $k^f(l,l')$.
% %
% Then, the predictive distribution is an $L$-dimensional Gaussian $\cN(\*\mu(\*x), \*\Sigma(\*x))$ with
% \begin{align*}
%  \*\mu(\*x) 
%  % &
%  =
%  \*g(\*x)^\top
%  %  \left( \*K^f \otimes \*k(\*x) \right)
%  \left(
%  \*G
%  % \*K^f \otimes \*K
%  + \sigma^2_{\rm noise} \*I
%  \right)^{-1}
%  \*Y,
%  \
%  % \\
%  \*\Sigma(\*x) 
%  % &
%  =
%  \*K^f \cdot k(\*x,\*x)
%  -
%  \*g(\*x)^\top
%  %  \left( \*K^f \otimes \*k(\*x) \right)
%  \left(
%  \*G
%  % \*K^f \otimes \*K
%  + \sigma^2_{\rm noise} \*I
%  \right)^{-1}
%  \left(
%  \*K^f \otimes \*k(\*x)
%  \right),
% \end{align*}
% where 
% $\*Y \coloneqq (\*y_1^\top, \ldots, \*y_n^\top)^\top \in \RR^{n L}$,
% $\*g(\*x) \coloneqq \*K^f \otimes \*k(\*x)$,
% and
% $ \*G \coloneqq \*K^f \otimes \*K$ with Kronecker product $\otimes$.
% %
% We here describe an extension of PFES for this correlated GPR objectives.

For the coupled setting, we need to evaluate analytically intractable integrations in \eqref{eq:Z} and \eqref{eq:conditional-entropy-ind}.
The normalization constant $Z$ \eqref{eq:Z} is defined by the sum of the integral of Gaussian distribution on the hyper-rectangle region ($\cC_m$).
The numerical computation of this form of integrations have been extensively studied \citep{Genz2009-Computation} mainly in the context of the Gaussian probability calculation. 
%
% Current standard packages (e.g., xx) even can evaluate the CDF of more than hundreds of dimensions with reasonable accuracy and computational time.
%
The integration in the entropy \eqref{eq:conditional-entropy-ind} can also be evaluated through the Gaussian probability (Appendix~\ref{app:moment-TN} shows computational detail).
Although this approach requires $O(L)$ times $L-1$-dimensional and $O(L^2)$ times $L-2$-dimensional CDF calculations, in many practical problems, the number of objectives $L$ is quite small.
% this is feasible enough for most of cases because $L$ is usually \red{less than $10$}.
% we employ this strategy as a default choice since $L$ is usually quite small.
%
% On the other hand, introducing some approximation is another possible approach. 
%
% For example, one of simple approaches is a moment matching based approximation for which we show detail in Appendix~\ref{}.

% $\int_{\cC_m} p(\*f_{\*x} \mid \cD) {\rm d} \*f_{\*x}$ is 

%
{
For the decoupled setting, if $L = 2$, we can derive a simple form of the entropy calculation because the conditional distribution $p(\*f^{{\setminus l}}_{\*x} \mid f^l_{\*x}, \cD)$ in \eqref{eq:marginal-given-frontier} becomes a one-dimensional Gaussian distribution.}
Let 
$\sigma_{12}^2(\*x)$ 
be the predictive covariance of two-dimensional $\*f_{\*x} = (f^1_{\*x}, f^2_{\*x})^\top$.
Then, we obtain the following theorem:
% --------------------------------------------------
% Correlated L = 2
% --------------------------------------------------
\begin{theo}
 \label{thm:single-entropy-corr-L-eq-2}
 % Let
 % % \begin{align*}	       
 % $W_{m'2} \coloneqq	       
 % \Phi
 % %\left(
 % (
 % % \frac
 % ({f_2^{m'+1} - u(\*x)})/{s(\*x)}
 % % \right)
 % )
 % -
 % \Phi
 % (
 % % \left(
 % %\frac
 % ({f_2^{m'} - u(\*x)})/{s(\*x)}
 % % \right)
 % )
 % $, where 
 % $u(\*x) \coloneqq \frac	{\sigma^2_{12}(\*x) ( f^2_{\*x} - \mu_2(\*x) )}	{\sigma^2_2(\*x)} + \mu_1(\*x)$ and
 % $s^2(\*x) \coloneqq \sigma^2_1(\*x) - \frac{(\sigma^2_{12}(\*x))^2}{\sigma^2_2(\*x)}$.	
 %   % \end{align*}
 % Then, 
 % \begin{align*}
 %  H[ p(f_1(\*x) \mid \cD, \*f_{\*x} \preceq \cF^*) ]
 %  & =
 %  - 
 %  \sum_{m = 0}^{M_1}
 %  \frac{\sum_{m' \in \cM_{1,m}} W_{m'2}}{Z}	
 %  \tilde{Z}_{m1}
 %  \Biggl(
 %  \log \frac{\sum_{m' \in \cM_{1,m}} W_{m'2}}{Z}	
 %  - % \tilde{Z}_{m1}
 %  \log
 %  (\sqrt{2 \pi e} \sigma_{1}(\*x) ) % \tilde{Z}_{m1})
 %  -
 %  \tilde{\Gamma}_{ml}
 %  % \frac{ \tilde{\alpha}_1^{m+1} \phi(\tilde{\alpha}_1^{m+1})  - \tilde{\alpha}_1^{m} \phi(\tilde{\alpha}_1^{m}) }{ 2 \tilde{Z}_{m1}}
 %  % \\
 %  % &	 \qquad
 %  % +
 %  % \tilde{Z}_{m1} 
 %  % \log \tilde{Z}_{m1}	
 %  % \tilde{Z}_{m1}
 %  \Biggr).	 
 % \end{align}
 Let
 $W_{m'2}(f^1_{\*x}) \coloneqq	       
 \Phi
 \left(
 \frac{f_2^{m'+1} - u(\*x \mid f^1_{\*x})}{s(\*x)}
 \right)
 -
 \Phi
 \left(
 \frac{f_2^{m'} - u(\*x \mid f^1_{\*x})}{s(\*x)}
 \right)$, 
 where 
 $u(\*x \mid f) \coloneqq \frac {\sigma^2_{12}(\*x) ( f - \mu_1(\*x) )}	{\sigma^2_1(\*x)} + \mu_2(\*x)$ 
 and
 $s^2(\*x) \coloneqq \sigma^2_2(\*x) - \frac{(\sigma^2_{12}(\*x))^2}{\sigma^2_1(\*x)}$.
 For the two dimensional correlated GPRs $\*f(\*x)$, the entropy of $p(f^1_{\*x} \mid \cD, \*f_{\*x} \preceq \cF^*)$ is given by
 \begin{align*}
  H[ p(f^1_{\*x} \mid \cD, \*f_{\*x} \preceq \cF^*) ]
  =
  -	       
  \sum_{s = 0}^{S_1 - 1}
  \int_{\tilde{f}_1^{s}}^{\tilde{f}_1^{s+1}}
  \left(
  \frac{
  {\phi(\alpha_{\*x}^1)}
  % p(f_1(\*x) \mid \cD)
  }{Z}	
  \sum_{m' \in \cM(1,s)}
  W_{m'2}(f^1_{\*x})
  \right)
  \log
  \left(
  \frac{
  {\phi(\alpha_{\*x}^1)}
  % p(f_1(\*x) \mid \cD)
  }{Z}	
  \sum_{m' \in \cM(1,s)}
  W_{m'2}(f^1_{\*x})
  \right)
  {\rm d} f^1_{\*x} 
\end{align*} 
where 
$\alpha_{\*x}^1 \coloneqq ( f_{\*x}^1 - \mu_1(\*x) )/ \sigma_1(\*x)$.
\end{theo}
\begin{proof}
% --------------------------------------------------
% \section{Proof of Theorem~\ref{thm:single-entropy-corr-L-eq-2}}
Let
$u(\*x \mid f) \coloneqq \frac {\sigma^2_{12}(\*x) ( f - \mu_1(\*x) )}	{\sigma^2_1(\*x)} + \mu_2(\*x)$ 
and
 $s^2(\*x) \coloneqq \sigma^2_2(\*x) - \frac{(\sigma^2_{12}(\*x))^2}{\sigma^2_1(\*x)}$.	
From \eqref{eq:marginal-given-frontier}, the marginal can be written as follows:
\begin{align*}
 % p(f^1_{\*x} \in (\tilde{f}_l^s, \tilde{f}_l^{s+1}]  \mid \cD, \*f_{\*x} \preceq \cF^*) 
 p(f^1_{\*x} \mid \cD, \*f_{\*x} \preceq \cF^*) 
  & =
  \frac{p(f^1_{\*x} \mid \cD)}{Z}	
  \sum_{m' \in \cM(1,s_1^{(f^1_{\*x})}) }
  % \int_{\cC_{m'}^{{\setminus l}}}
 % \int_{f_2^{m'}}^{f_2^{m'+1}}
 \int_{\ell^2_{m'}}^{u^2_{m'}}
  p(f^{2}_{\*x} \mid f^1_{\*x}, \cD) 
  {\rm d} f^2_{\*x}		       
  \\ % -------------------------
  & =
  \frac{p(f^1_{\*x} \mid \cD)}{Z}	
 % \sum_{m' \in \cM_{1,m}}
 \sum_{m' \in \cM(1,s_1^{(f^1_{\*x})}) }
  W_{m'2}(f^1_{\*x})
  % \int_{f_2^{m'}}^{f_2^{m'+1}} 	       p(f_{1}(\*x) \mid f_2(\*x), \cD) 	       {\rm d} \*f_2(\*x)      
\end{align*}
where
\begin{align*}	       
  W_{m'2}(f^1_{\*x}) \coloneqq	       
  \Phi
  \left(
 % \frac{f_2^{m'+1} - u(\*x \mid f^1_{\*x})}{s(\*x)}
 \frac{u^2_{m'} - u(\*x \mid f^1_{\*x})}{s(\*x)}
  \right)
  -
  \Phi
  \left(
 % \frac{f_2^{m'} - u(\*x \mid f^1_{\*x})}{s(\*x)}
 \frac{\ell^2_{m'} - u(\*x \mid f^1_{\*x})}{s(\*x)}
  \right)
\end{align*}
Then, the entropy is 
 \begin{align*}
  &
  H[ p(f^1_{\*x} \mid \cD, \*f_{\*x} \preceq \cF^*) ]
  \\
  & =
  - 
  \sum_{s = 0}^{S_1 - 1}
  \int_{\tilde{f}_1^{s}}^{\tilde{f}_1^{s+1}}	
  % p(f^1_{\*x} \in (\tilde{f}_l^s, \tilde{f}_l^{s+1}]  \mid \cD, \*f_{\*x} \preceq \cF^*) 
  p(f^1_{\*x} \mid \cD, \*f_{\*x} \preceq \cF^*) 
  \log
  p(f^1_{\*x} \mid \cD, \*f_{\*x} \preceq \cF^*) 
  {\rm d} f^1_{\*x}
  \\
  & =
  - 
  \sum_{s = 0}^{S_1 - 1}
  \int_{\tilde{f}_1^{s}}^{\tilde{f}_1^{s+1}}
  \left(
  \frac{p(f^1_{\*x} \mid \cD)}{Z}	
  \sum_{m' \in \cM(1,s)}
  % \sum_{m' \in \cM(1,s_1^{(f^1_{\*x})}) }
  W_{m'2}(f^1_{\*x})
  \right)
  \log
  \left(
  \frac{p(f^1_{\*x} \mid \cD)}{Z}	
  \sum_{m' \in \cM(1,s)}
  % \sum_{m' \in \cM(1,s_1^{(f^1_{\*x})}) }
  W_{m'2}(f^1_{\*x})
  \right)
  {\rm d} f^1_{\*x}
  \\ % -------------------------
  & =
  {
  -	       
  \sum_{s = 0}^{S_1 - 1}
  \int_{\tilde{f}_1^{s}}^{\tilde{f}_1^{s+1}}
  \left(
  \frac{
  \phi(\alpha_{\*x}^1)
  % p(f_1(\*x) \mid \cD)
  }{Z}	
  % \sum_{m' \in \cM_{1,m}}
  \sum_{m' \in \cM(1,s)}
  W_{m'2}(f^1_{\*x})
  \right)
  \log
  \left(
  \frac{
  \phi(\alpha_{\*x}^1)
  % p(f_1(\*x) \mid \cD)
  }{Z}	
  % \sum_{m' \in \cM_{1,m}}
  \sum_{m' \in \cM(1,s)}
  W_{m'2}(f^1_{\*x})
  \right)
  {\rm d} f^1_{\*x}
  }
 \end{align*}
where 
$\alpha_{\*x}^1 \coloneqq ( f_{\*x}^1 - \mu_1(\*x) )/ \sigma_1(\*x)$.
\end{proof}
\noindent
Although the integral inside the sum is analytically intractable, we can numerically calculate it easily because the integral is over the one-dimensional interval.

{
In the case of $L > 2$, % $p(f^l_{\*x} \in (\tilde{f}_l^m, \tilde{f}_l^{m+1}] \mid \cD, \*f_{\*x} \preceq \cF^*)$ 
% defined by \eqref{eq:marginal-given-frontier} is analytically intractable unlike the case of $L = 2$
the integral 
$\int_{\cC_{m'}^{{\setminus l}}} p(\*f^{{\setminus l}}_{\*x} \mid f_{\*x}^l, \cD) {\rm d} \*f^{{\setminus l}}_{\*x}$
in \eqref{eq:marginal-given-frontier} 
is also the multi-dimensional Gaussian integration \citep{Genz2009-Computation}.
% is analytically intractable unlike the case of $L = 2$.
% we 
% $\int_{f^l_{\*x}}
% p(f^l_{\*x} \in (\tilde{f}_l^m, \tilde{f}_l^{m+1}] \mid \cD, \*f_{\*x} \preceq \cF^*) \log
% p(f^l_{\*x} \in (\tilde{f}_l^m, \tilde{f}_l^{m+1}] \mid \cD, \*f_{\*x} \preceq \cF^*)
% {\mathrm d} f^l_{\*x}$ 
%
The marginal density
$p(f^l_{\*x} \mid \cD, \*f_{\*x} \preceq \cF^*)$ 
defined by \eqref{eq:marginal-given-frontier} can 
also be evaluated through the integration of the Gaussian density because 
$p(\*f^{{\setminus l}}_{\*x} \mid f_{\*x}^l, \cD)$
can be analytically derived for a given $f_{\*x}^l$.
Here again, for the integral in \eqref{eq:marginal-given-frontier}, we can use numerical technique for the Gaussian probability \citep{Genz2009-Computation}.
Then, we can simply approximate the integral of the entropy
$\int_{f^l_{\*x}}
p(f^l_{\*x} \mid \cD, \*f_{\*x} \preceq \cF^*) \log
p(f^l_{\*x} \mid \cD, \*f_{\*x} \preceq \cF^*)
{\mathrm d} f^l_{\*x}
% \approx
% \sum_{}
% p(f^l_{\*x} \in (\tilde{f}_l^m, \tilde{f}_l^{m+1}] \mid \cD, \*f_{\*x} \preceq \cF^*) \log
% p(f^l_{\*x} \in (\tilde{f}_l^m, \tilde{f}_l^{m+1}] \mid \cD, \*f_{\*x} \preceq \cF^*)
$ 
by a sum of finite grid points.
This is also one-dimensional integral, and thus accurate approximation can be expected.
}

% The detailed definition of $\Psi^{{\setminus l}}_{m'}$ is shown in \red{Appendix~xx}, \red{which has almost the same form as \eqref{eq:Psi}}.
%
%Then, we can obtain an approximation of the entropy as follows
% --------------------------------------------------
% Decoupled Correlated L > 2
% --------------------------------------------------
% \begin{align}
%  H[ p(f^l_{\*x} \mid \cD, \*f_{\*x} \preceq \cF^*) ]
%  \approx
%  \red{\text{ToWrite}}
%  \label{eq:single-entropy-corr}
% \end{align}

% \clearpage

% --------------------------------------------------
\section{Entropy Evaluation for Correlated Objectives}
\label{app:moment-TN}

Here, we redefine
\begin{align*}
 Z \coloneqq
 \int_{\cF}
 p(\*f_{\*x} \mid \cD) {\mathrm d} \*f_{\*x},
 \text{ and }
 Z_m \coloneqq
 \int_{\cC_m}
 p(\*f_{\*x} \mid \cD) {\mathrm d} \*f_{\*x},
\end{align*}
which indicate $Z = \sum_{m=1}^M Z_m$.
The entropy of the conditional distribution
$p(\*f_{\*x} \mid \cD, \*f_{\*x} \preceq \cF^*)$
can be transformed as follows:
\begin{align}
 H[p(\*f_{\*x} \mid \cD, \*f_{\*x} \preceq \cF^*)]
 & =
 -
 \int_{\cF}
 \frac{p(\*f_{\*x} \mid \cD)}{Z}
 \log
 \frac{p(\*f_{\*x} \mid \cD)}{Z}
 {\mathrm d} \*f_{\*x}
 \nonumber
 \\ % -------------------------
 & =
 -
 \frac{1}{Z}
 \sum_{m = 1}^M
 \int_{\cC_m}
 p(\*f_{\*x} \mid \cD)
 \log
 p(\*f_{\*x} \mid \cD)
 {\mathrm d} \*f_{\*x}
 +
 \log Z
 \nonumber
 \\ % -------------------------
 & =
 \frac{1}{Z}
 \sum_{m = 1}^M
 Z_m
 \left(
 \underbrace{
 -
 \int_{\cC_m}
 \frac{p(\*f_{\*x} \mid \cD)}{Z_m}
 \log
 \frac{p(\*f_{\*x} \mid \cD)}{Z_m}
 {\mathrm d} \*f_{\*x}
 }_{(\star)}
 +
 \log Z_m
 \right)
 -
 \log Z
 \label{eq:entropy-given-frontier-corr}
\end{align}
The term indicated by $\star$ is the entropy of the multi-variate truncated normal distribution.
This term can also be written as
$- \EE_{\rm TN}
 \left[
 \log
 \frac{p(\*f_{\*x} \mid \cD)}{Z_m}
 \right]$,
where
$\EE_{\rm TN}$
is an expectation by the truncated normal distribution
\begin{align*}
 p(\*f_{\*x} \mid \cD, \*f_{\*x} \in \cC_m)
 =
 \begin{cases}
  % p(\*f_{\*x} \mid \cD), & \text{if } \*f_{\*x} \in \cC_m \\
  \cN(\*\mu,\*\Sigma) / Z_m, & \text{if } \*f_{\*x} \in \cC_m, \\
  0, & \text{otherwise},
 \end{cases}
\end{align*}
with {the predictive mean $\*\mu \in \RR^L$ and the predictive covariance matrix $\*\Sigma \in \RR^{L \times L}$ of the current GPR.}
We derive that this entropy can be represented through the moment of the truncated normal distribution:
\begin{align}
 % -
 % \int_{\cC_m}
 % \frac{p(\*f_{\*x} \mid \cD)}{Z_m}
 % \log
 % \frac{p(\*f_{\*x} \mid \cD)}{Z_m}
 % {\mathrm d} \*f_{\*x}
 % =
 - \EE_{\rm TN}
 \left[
 \log
 \frac{p(\*f_{\*x} \mid \cD)}{Z_m}
 \right]
 & =
 - \EE_{\rm TN}
 \left[
 \log
 {p(\*f_{\*x} \mid \cD)}
 % \cN(\*\mu,\*\Sigma)
 \right]
 +
 \log
 {Z_m}
 \nonumber
 \\ %-------------------------
 & =
 -
 \EE_{\rm TN}
 \left[
 - \frac{1}{2} \log | 2 \pi \*\Sigma |
 - \frac{1}{2} (\*f_{\*x} - \*\mu)^\top \*\Sigma^{-1} (\*f_{\*x} - \*\mu)
 \right]
 +
 \log
 {Z_m}
 \nonumber
 \\ %-------------------------
 & =
 \frac{1}{2} \log | 2 \pi \*\Sigma |
 +
 \frac{1}{2}
 \EE_{\rm TN}
 \left[
 (\*f_{\*x} - \*\mu)^\top \*\Sigma^{-1} (\*f_{\*x} - \*\mu)
 \right]
 +
 \log
 {Z_m}
 \label{eq:truncated-ent}
\end{align}
Let
$\*\mu_{\rm TN} \coloneqq \EE_{\rm TN}[ \*f_{\*x} ]$,
$\*d \coloneqq \*\mu_{\rm TN} - \*\mu$,
and
$\*\Sigma_{\rm TN} \coloneqq \EE_{\rm TN}[ (\*f_{\*x} - \*\mu_{\rm TN}) (\*f_{\*x} - \*\mu_{\rm TN})^\top ]$.
Then, the second term of the above equation (\ref{eq:truncated-ent}) is written as
\begin{align}
 \EE_{\rm TN}
 \left[
 (\*f_{\*x} - \*\mu)^\top \*\Sigma^{-1} (\*f_{\*x} - \*\mu)
 \right]
 & =
 \mathrm{Trace}\left(
 \*\Sigma^{-1}
 \EE_{\rm TN}
 \left[
 (\*f_{\*x} - \*\mu)
 (\*f_{\*x} - \*\mu)^\top
 \right]
 \right)
 \nonumber
 \\ %-------------------------
 & =
 \mathrm{Trace}\left(
 \*\Sigma^{-1}
 \EE_{\rm TN}
 \left[
 (\*f_{\*x} - \*\mu_{\rm TN} + \*d)
 (\*f_{\*x} - \*\mu_{\rm TN} + \*d)^\top
 \right]
 \right)
 \nonumber
 \\ %-------------------------
 & =
 \mathrm{Trace}\left(
 \*\Sigma^{-1}
 (
 \EE_{\rm TN}
 \left[
 (\*f_{\*x} - \*\mu_{\rm TN})
 (\*f_{\*x} - \*\mu_{\rm TN})^\top
 \right]
 +
 \EE_{\rm TN}
 \left[
 \*d
 \*d^\top
 \right]
 )
 \right)
 \nonumber
 \\ %-------------------------
 & =
 \mathrm{Trace}\left(
 \*\Sigma^{-1}
 (
 \*\Sigma_{\rm TN}
 +
 \*d
 \*d^\top
 )
 \right)
 \label{eq:trunc_ex_squared_form}
\end{align}
If $\*\mu_{\rm TN}$ and $\*\Sigma_{\rm TN}$ are available, the entropy \eqref{eq:entropy-given-frontier-corr} can be evaluated by combining \eqref{eq:truncated-ent} and \eqref{eq:trunc_ex_squared_form}.

{The two expected values $\*\mu_{\rm TN}$ and $\*\Sigma_{\rm TN}$ can be obtained from the first and the second moment of the multi-variate truncated normal distribution, for which \citet{Manjunath2009-Moments} show efficient computations through the Gaussian integral calculation.}
Let $\mu^i$ and $\mu^i_{TN}$ be the $i$-th element of $\bm{\mu}$ and $\bm{\mu}_{\rm TN}$, respectively, and let $\sigma^{i, j}$ and $\sigma^{i, j}_{\rm TN}$ be the $(i, j)$-th element of $\bm{\Sigma}$ and $\bm{\Sigma}_{\rm TN}$, respectively.
% $a_k, b_k$$B$r$=$l$>$l(B$\bm{a}, \bm{b}$$B$N(B$k$$BHVL\$NMWAG$H$9$k$H(B,
%
By defining the $k$-th dimensional marginal distribution of the truncated normal as
\begin{align*}
 F_k(f^k_{\*x}) &= p( f^k_{\*x} \mid \cD, \*f_{\*x} \in \cC_m) \\
 &= \int^{u^1_m}_{\ell^1_{m}} \dots \int^{u^{k-1}_m}_{\ell^{k-1}_{m}}  \int^{u^{k+1}_m}_{\ell^{k+1}_{m}} \dots \int^{u^L_m}_{\ell^L_{m+1}}
 % TN(\bm{f}\mid \bm{\mu}_{TN}, \bm{\Sigma_{TN}})
 p(\*f_{\*x} \mid \cD, \*f_{\*x} \in \cC_m)
 \mathrm{d}\bm{f}^{\backslash k}_{\*x},
\end{align*}
$\mu^i_{\rm TN}$ can be represented as
\begin{align}
 \mu^i_{\rm TN} &= \mu^i + \sum_{k=1}^L
 \sigma^{i, k}\bigl( F_k(\ell_m^k) - F_k(u_m^k) \bigl),
 \nonumber \\
 % \sigma^{i, k}\bigl( F_k(a_k) - F_k(b_k) \bigl), \\
 &= \mu^i + d^i,
 % \sum_{k=1}^L \sigma^{i, k}\bigl( F_k(a_k) - F_k(b_k) \bigl)
 % d^i &= \sum_{k=1}^L \sigma^{i, k}\bigl( F_k(f^{m}_k) - F_k(f^{m+1}_k) \bigl).
 \label{eq:mu_TN}
\end{align}
where
$d^i \coloneqq \sum_{k=1}^L \sigma^{i, k}\bigl( F_k(\ell_{m}^k) - F_k(u_{m}^k) \bigl)$.
For $\sigma^{i, j}_{\rm TN}$,
by defining the $(k,q)$-th two dimensional marginal distribution of the truncated normal as
\begin{align*}
 F_{k, q}(f^k_{\*x}, f^q_{\*x})
 % F_{k, q}(f_k, f_q)
 &= p( f^k_{\*x}, f^q_{\*x} \mid \cD, \*f_{\*x} \in \cC_m) \\
 &=
 \int^{u^1_m}_{\ell^1_{m}}
 \dots
 \int^{u^{k-1}_m}_{\ell^{k-1}_{m}}
 \int^{u^{k+1}_m}_{\ell^{k+1}_{m}}
 \dots
 \int^{u^{q-1}_m}_{\ell^{q-1}_{m}}
 \int^{u^{q+1}_m}_{\ell^{q+1}_{m}}
 \dots
 \int^{u^L_m}_{\ell^L_{m}}
 % = \int^{b_1}_{a_1} \dots, \int^{b_{k-1}}_{a_{k-1}} \int^{b_{k+1}}_{a_{k+1}}
 % \dots \int^{b_{q-1}}_{a_{q-1}} \int^{b_{q+1}}_{a_{q+1}} \dots \int^{b_L}_{a_L}
 % TN(\bm{f}\mid \bm{\mu}_{TN}, \bm{\Sigma_{TN}})
 p(\*f_{\*x} \mid \cD, \*f_{\*x} \in \cC_m)
 \mathrm{d} \bm{f}^{\backslash k, q}_{\*x},
\end{align*}
we obtain
\begin{align}
 \sigma^{i, j}_{TN}
 &=
 \sigma^{i, j} - d^i d^j +
 \sum_{k=1}^L \sigma^{i, k}\frac{\sigma^{j, k}
 \bigl( \ell_m^k F_k(\ell^k_m) - u^{m}_k F_k(u^k_{m}) \bigl)}{\sigma^{k, k}}
 \nonumber
 \\
 &+
 \sum_{k=1}^L \sigma^{i, k}
 \sum_{q \neq k} \biggl( \sigma^{j, q} - \frac{\sigma^{k, q}\sigma^{j, k}}{\sigma^{k, k}}\biggl)
 \nonumber
 \\
 & \qquad \cdot
 \biggl[ \bigl( F_{k, q}(\ell_m^k, \ell_m^q) - F_{k, q}(\ell_m^k, u_{m}^q)\bigl) - \bigl( F_{k, q}(u_{m}^k, \ell_m^q) - F_{k, q}(u_{m}^k, u_{m}^q)\bigl)\biggl].
 \label{eq:Sigma_TN}
\end{align}
Thus, to calculate \eqref{eq:mu_TN} and \eqref{eq:Sigma_TN}, $O(L)$ times $L - 1$ dimensional Gaussian integration and $O(L^2)$ times $L - 2$ dimensional Gaussian integration are necessary.

% --------------------------------------------------
\section{Acquisition Function Computation}
\label{app:time}

We randomly selected $50, 100$, and $200$ training instances, and calculated each acquisition function for randomly selected $100$ points.
We measured CPU time on our python code by the single thread execution.
Precise evaluation of computational cost is difficult because of its dependence on implementation detail.
Our main purpose here is to show PFES is feasible enough for reasonable size of $L$.
%
% We used DTLZ3 and DTLZ4, because they have $L = 4$ which is the largest value among four datasets we used in Section~\ref{sec:benchmark}. 
%
% The results are shown in \tablename~\ref{tab:time-DTLZ3} and \ref{tab:time-DTLZ4}.
The results are shown in \tablename~\ref{tab:time-AS}-\ref{tab:time-DTLZ4} (OOM indicates out-of-memory).

Overall, the results have the same tendency as \tablename~\ref{tab:time} in our main text. 
Note that since Ackley/Sphere 2D and ZDT4 are $L = 2$, \#cells in PFES is always $50$, which is equal to $|\mathrm{PF}|$.

\begin{table}[t]
\caption{Computational time of acquisition function on Ackley/Sphere 2D.}
 \label{tab:time-AS}
 \centering
\begin{tabular}{llrrr}
{}                     &{}            & 50 train data    & 100 train data    & 200 train data    \\
\hline \hline
ParEGO                 &              & 0.44 $\pm$ 0.01   & 0.67 $\pm$ 0.02    & 1.32 $\pm$ 0.07    \\
\hline
SMSego                 &              & 0.20 $\pm$ 0.01   & 0.35 $\pm$ 0.01    & 0.62 $\pm$ 0.02    \\
\hline
EHI                    &              & 0.10 $\pm$ 0.00   & 0.13 $\pm$ 0.00    & 0.13 $\pm$ 0.00   \\
\hline
                MESMO  &         & 48.35 $\pm$ 1.27   & 48.41 $\pm$ 0.85    & 48.22 $\pm$ 0.66    \\
                % MESMO  & Total        & 48.35 $\pm$ 1.27   & 48.41 $\pm$ 0.85    & 48.22 $\pm$ 0.66    \\
                %        & RFM          & 0.06 $\pm$ 0.00 & 0.06 $\pm$ 0.00  & 0.07 $\pm$ 0.00\\
                %        & NSGAII       & 48.10 $\pm$ 1.29   & 48.14 $\pm$ 0.87    & 47.96 $\pm$ 0.67    \\
                %        & eval entropy & 0.36 $\pm$ 0.02   & 0.43 $\pm$ 0.03    & 0.44 $\pm$ 0.020    \\
\hline
                PFES   & Total        & 39.42 $\pm$ 0.65   & 41.20 $\pm$ 0.54    & 43.81 $\pm$ 0.86    \\
                       & RFM          & 0.05 $\pm$ 0.00 & 0.060 $\pm$ 0.00  & 0.06 $\pm$ 0.00 \\
                       & NSGAII       & 38.97 $\pm$ 0.67   & 40.67 $\pm$ 0.56    & 43.26 $\pm$ 0.88    \\
                       & QHV          & 0.040 $\pm$ 0.00 & 0.05 $\pm$ 0.00  & 0.04 $\pm$ 0.00  \\
                       & eval entropy & 0.36 $\pm$ 0.02   & 0.43 $\pm$ 0.03    & 0.44 $\pm$ 0.02    \\
                       & \# Cell       & 50.00 $\pm$ 0.00   & 50.00 $\pm$ 0.00    & 50.00 $\pm$ 0.00   
\end{tabular}
\end{table}

% --------------------------------------------------
% ZDT4
% --------------------------------------------------
\begin{table}[t]
 \caption{Computational time of acquisition function on ZDT4.}
 \label{tab:time-ZDT4}
 \centering
\begin{tabular}{llrrr}
           &              & 50 train data    & 100 train data    & 200 train data    \\
\hline \hline
ParEGO                 &              & 0.40 $\pm$ 0.05    & 0.55 $\pm$ 0.06     & 1.23 $\pm$ 0.02    \\
\hline
SMSego                 &              & 0.24 $\pm$ 0.02    & 0.37 $\pm$ 0.03     & 0.71 $\pm$ 0.03    \\
\hline
EHI                    &              & 0.11 $\pm$ 0.00    & 0.11 $\pm$ 0.00     & 0.13 $\pm$ 0.00   \\
\hline
                MESMO  &         & 51.33 $\pm$ 0.47   & 48.40 $\pm$ 0.52    & 44.17 $\pm$ 0.63    \\
%                 MESMO  & Total        & 51.33 $\pm$ 0.47   & 48.40 $\pm$ 0.52    & 44.17 $\pm$ 0.63    \\
                       % & RFM          & 0.07 $\pm$ 0.01    & 0.07 $\pm$ 0.00     & 0.06 $\pm$ 0.00  \\
                       % & NSGAII       & 51.04 $\pm$ 0.47   & 48.11 $\pm$ 0.53    & 43.92 $\pm$ 0.64    \\
                       % & eval entropy & 0.23 $\pm$ 0.01    & 0.23 $\pm$ 0.01     & 0.19 $\pm$ 0.01    \\
\hline
                PFES   & Total        & 51.59 $\pm$ 0.44   & 48.71 $\pm$ 0.35    & 51.89 $\pm$ 0.16    \\
                       & RFM          & 0.06  $\pm$ 0.00   & 0.06  $\pm$ 0.00    & 0.07  $\pm$ 0.00   \\
                       & NSGAII       & 50.96 $\pm$ 0.44   & 48.13 $\pm$ 0.37    & 51.26 $\pm$ 0.16    \\
                       & QHV          & 0.05  $\pm$ 0.00   & 0.05  $\pm$ 0.00    & 0.05  $\pm$ 0.00   \\
                       & eval entropy & 0.51  $\pm$ 0.02   & 0.47  $\pm$ 0.03    & 0.51 $\pm$ 0.01   \\
                       & \# Cell      & 50.00 $\pm$ 0.00   & 50.00 $\pm$ 0.00    & 50.00 $\pm$ 0.00   
\end{tabular}
\end{table}

% --------------------------------------------------
% DTLZ3
% --------------------------------------------------
\begin{table}[t]
 \caption{Computational time of acquisition function on DTLZ3}
 \label{tab:time-DTLZ3}
 \centering
\begin{tabular}{llrrr}
                       &              & 50 train data    & 100 train data     & 200 train data    \\
\hline \hline
ParEGO                 &              & 0.51 $\pm$ 0.01    & 0.68 $\pm$ 0.03      & 1.35 $\pm$ 0.04    \\
\hline
SMSego                 &              & 1.58 $\pm$ 0.61    & 2.36 $\pm$ 0.90      & 2.75 $\pm$ 0.61     \\
\hline
EHI                    &              & 6.26 $\pm$ 5.76    & 21.19 $\pm$ 30.11    & 20.21 $\pm$ 11.03  \\
\hline
                MESMO  &         & 62.88 $\pm$ 1.39   & 49.38 $\pm$ 0.72     & 62.25 $\pm$ 0.93    \\
                % MESMO  & Total        & 62.88 $\pm$ 1.39   & 49.38 $\pm$ 0.72     & 62.25 $\pm$ 0.93    \\
                %        & RFM          & 0.11 $\pm$ 0.00    & 0.11 $\pm$ 0.00      & 0.13 $\pm$ 0.00  \\
                %        & NSGAII       & 62.51 $\pm$ 1.40   & 49.057 $\pm$ 0.74    & 61.82 $\pm$ 0.93    \\
                %        & eval entropy & 0.26 $\pm$ 0.02    & 0.22 $\pm$ 0.02      & 0.29 $\pm$ 0.01   \\
\hline
                PFES   & Total        & 65.36 $\pm$ 2.02   & 57.73 $\pm$ 1.26     & 64.97 $\pm$ 1.39    \\
                       & RFM          & 0.11 $\pm$ 0.00    & 0.11 $\pm$ 0.00      & 0.13 $\pm$ 0.00   \\
                       & NSGAII       & 63.024 $\pm$ 2.07  & 55.35 $\pm$ 1.25     & 62.23 $\pm$ 1.40    \\
                       & QHV          & 1.12 $\pm$ 0.12    & 1.15 $\pm$ 0.16      & 1.32 $\pm$ 0.09    \\
                       & eval entropy & 1.10 $\pm$ 0.09   & 1.12 $\pm$ 0.07     & 1.29 $\pm$ 0.08    \\
                       & \# Cell      & 570.99 $\pm$ 87.26 & 614.56 $\pm$ 109.27  & 662.69 $\pm$ 107.33
\end{tabular}
\end{table}

% --------------------------------------------------
% DTLZ4
% --------------------------------------------------
\begin{table}[t]
 \centering
\caption{Computational time of acquisition function on DTLZ4}
 \label{tab:time-DTLZ4}
\begin{tabular}{llrrr}
                       &              & 50 train data     & 100 train data    & 200 train data                        \\
\hline \hline
ParEGO                 &              & 0.53 $\pm$ 0.01     & 0.64 $\pm$ 0.02     & 1.44 $\pm$ 0.04                         \\
\hline
SMSego                 &              & 6.32 $\pm$ 1.50     & 11.83 $\pm$ 1.92    & 19.86 $\pm$ 2.89                        \\
\hline
EHI                    &              & 317.55 $\pm$ 63.18  & 796.10 $\pm$ 199.66 & OOM                                   \\
\hline
                MESMO  &         & 59.90 $\pm$ 0.43    & 62.89 $\pm$ 0.61    & 65.30 $\pm$ 0.93                        \\
                % MESMO  & Total        & 59.90 $\pm$ 0.43    & 62.89 $\pm$ 0.61    & 65.30 $\pm$ 0.93                        \\
                %        & RFM          & 0.11 $\pm$ 0.00     & 0.13 $\pm$ 0.00     & 0.14 $\pm$ 0.00                         \\
                %        & NSGAII       & 59.53 $\pm$ 0.44    & 62.45 $\pm$ 0.62    & 64.84 $\pm$ 0.93                        \\
                %        & eval entropy & 0.27 $\pm$ 0.01     & 0.30 $\pm$ 0.02     & 0.32 $\pm$ 0.01                         \\
\hline
                PFES   & Total        & 62.36 $\pm$ 0.72    & 60.97 $\pm$ 0.36    & 67.85 $\pm$ 0.86                        \\
                       & RFM          & 0.11 $\pm$ 0.00     & 0.11 $\pm$ 0.00     & 0.14 $\pm$ 0.00                         \\
                       & NSGAII       & 59.85 $\pm$ 0.79    & 58.39 $\pm$ 0.35    & 64.77 $\pm$ 0.97                        \\
                       & QHV          & 1.28 $\pm$ 0.08     & 1.23 $\pm$ 0.08     & 1.59 $\pm$ 0.24                         \\
                       & eval entropy & 1.13 $\pm$ 0.07     & 1.25 $\pm$ 0.04     & 1.35 $\pm$ 0.12                         \\
                       & \# Cell       & 647.86 $\pm$ 124.84 & 645.88 $\pm$ 116.01 & 710.44 $\pm$ 141.77
\end{tabular}
\end{table}

% --------------------------------------------------
%\section{Additional Information on Experiments}

% --------------------------------------------------
\section{Experimental Settings}
\label{app:settings}

For GPR, we used GPy (\url{https://sheffieldml.github.io/GPy/}).
%e
%
The noise term in GPR $\sigma_{\rm noise}$ is fixed at $10^{-4}$.
The marginal likelihood optimization of the Gaussian kernel parameter $\sigma > 0$ is performed by using gradient descent.
This optimization was performed at every iteration in the benchmark experiments.
For the material datasets, we first randomly selected 100 samples to optimize 
$\sigma$
and
$\sigma_{\rm noise}$
through the marginal likelihood optimization.
These values were fixed during the BO procedure.
Since the material datasets are noisy, we employed this approach for avoiding unstable behaviors of all the compared methods because of the unstable GPR hyper-parameters.
We implemented ParEGO and SMSego by ourselves.
For ParEGO, the weighting constant $\rho$ of the augmented Tchebycheff function is set $0.05$ as indicated by the original paper \citep{Knowles2006-ParEGO}.
For SMSego, the coefficient of lower confidence bound is set as $\beta_t = \Phi^{-1}(0.5 + 1/2^L)$ which is also indicated by the original paper \citep{Ponweiser2008-Multiobjective}.
%
% For PESMO, we used the code in (\url{https://github.com/HIPS/Spearmint/tree/PESM}).
% %
% Since the original implementation did not work in our environment, we implemented our own code of PESMO based on it.
% %
% For PESMO, about Pareto sampling, we used the same settings as PFES.
% %
% The number of sampling is $10$, and the number of basis of RFM is $500$.
For MESMO, about Pareto sampling, we used the same settings as PFES.
The number of sampling is $10$, and the number of basis of RFM is $500$. 

% --------------------------------------------------
% \subsection{Other Related Methods}
% \label{app:other-methods}

% \red{For PESMO, 
% the author implementation 
% cannot be installed to our environment, which 
% is 
% a standard 
% ...}

% % --------------------------------------------------
% \subsection{Additional Result on LLTO}
% \label{app:LLTO}

% \figurename~\ref{fig:LLTO_long} shows the results on LLTO with the larger limit of the cumulative cost (Here, PESMO and PESMO (decoupled) is not included because our implementation of them are not quite fast).
% %
% In this longer iteration result, the superior performance of PFES (decoupled) can also be confirmed.

% % --------------------------------------------------
% % LLTO
% % --------------------------------------------------
% \begin{figure}[t]
%  \centering 
%  \igr{0.48}{./figs/result/resultsLLTO_1_5_long.pdf}
%  \caption{
%  Inference relative hyper-volume for LLTO with longer iterations.
%  }
%  \label{fig:LLTO_long}
% \end{figure}

\end{document}